%% file: ms.tex
\documentclass[sigconf]{acmart}

\AtBeginDocument{%
  }

\usepackage{booktabs} 
\usepackage{graphicx}
\usepackage{balance}  
\usepackage[ruled,vlined,linesnumbered]{algorithm2e}
\usepackage{algorithmic}
\usepackage{amsthm}
\usepackage{adjustbox}
\usepackage{xspace}
\usepackage{hyperref}
\usepackage{amsmath}
\usepackage{url}
\usepackage{autobreak}
\usepackage{color}
\usepackage{fancyvrb,xcolor}
\usepackage[center]{subfigure} 
\usepackage[mathscr]{euscript} 
\usepackage{makecell}
\usepackage{footnote} 
\usepackage{tablefootnote} 
\usepackage{mdwlist}
\usepackage{multirow}
\usepackage{enumitem}
\usepackage{lipsum}

\setcopyright{acmcopyright}
\copyrightyear{2023}
\acmYear{2023}
\setcopyright{acmlicensed}
\acmConference[KDD '23] {Proceedings of the 29th ACM SIGKDD Conference on Knowledge Discovery and Data Mining}{August 6--10, 2023}{Long Beach, CA, USA.}
\acmBooktitle{Proceedings of the 29th ACM SIGKDD Conference on Knowledge Discovery and Data Mining (KDD '23), August 6--10, 2023, Long Beach, CA, USA}
\acmPrice{15.00}
\acmISBN{979-8-4007-0103-0/23/08}
\acmDOI{10.1145/3580305.3599342}

\input{dfn}

\begin{document}

\title{Fast and Accurate Dual-Way Streaming PARAFAC2 for Irregular Tensors - Algorithm and Application}

\author{Jun-Gi Jang}
\email{elnino4@snu.ac.kr}
\affiliation{%
  \institution{Seoul National University}
  \city{Seoul}  
  \country{Republic of Korea}
}
\author{Jeongyoung Lee}
\email{ljklee@snu.ac.kr}
\affiliation{%
  \institution{Seoul National University}
  \city{Seoul}  
  \country{Republic of Korea}
}
\author{Yong-chan Park}
\email{wjdakf3948@snu.ac.kr}
\affiliation{%
  \institution{Seoul National University}
  \city{Seoul}  
  \country{Republic of Korea}
}
\author{U Kang}
\email{ukang@snu.ac.kr}
\affiliation{%
  \institution{Seoul National University}
  \city{Seoul}  
  \country{Republic of Korea}
}

\renewcommand{\shortauthors}{Jang, et al.}

\begin{abstract}
 \input{000abstract}
\end{abstract}



\maketitle

\vspace{-2mm}
\section{Introduction}
\label{sec:intro}
\input{010intro}

\section{Preliminaries and Problem Formulation}
\label{sec:prelim}
\input{020prelim}

\section{Proposed Method}
\label{sec:proposed}
\input{030proposed}


\section{Experiments}
\label{sec:experim}

\input{040experim}

\section{Related Works}
\label{sec:related}
\input{050related}

\section{Conclusion}
\label{sec:conclusion}
\input{060conclusion}

\begin{acks}
This work was supported by the National Research Foundation of Korea(NRF) funded by MSIT(2022R1A2C3007921).
This work was also supported by Institute of Information \& communications Technology Planning \& Evaluation(IITP) grant funded by the Korea government(MSIT) [No.2021-0-01343, Artificial Intelligence Graduate School Program (Seoul National University)], and [No.2021-0-02068, Artificial Intelligence Innovation Hub (Artificial Intelligence Institute, Seoul National University)].
The Institute of Engineering Research at Seoul National University provided research facilities for this work.
The ICT at Seoul National University provides research facilities for this study.
U Kang is the corresponding author.
\end{acks}

\bibliographystyle{ACM-Reference-Format}
\bibliography{mybib}

\appendix

\label{sec:appendix}
\input{080appendix}

\end{document}

%% file: dfn.tex
\newtheorem{problem}{Problem}

\newtheorem{definition}{Definition}
\newtheorem{lemma}{Lemma}
\newtheorem{theorem}{Theorem}

\SetKwComment{Comment}{$\triangleright$\ }{}

\newcommand{\method}{\textsc{Dash}\xspace}
\newcommand{\dpar}{\textsc{DPar2}\xspace}
\newcommand{\rd}{\textsc{RD}\xspace}
\newcommand{\als}{\textsc{PARAFAC2-ALS}\xspace}
\newcommand{\spartan}{\textsc{SPARTan}\xspace}
\newcommand{\spade}{\textsc{SPADE}\xspace}
\newcommand{\footnoteref}[1]{\textsuperscript{\ref{#1}}}

\newcommand{\hide}[1]{}

\newcommand{\mat}[1]{\mathbf{#1}}

\newcommand*{\QEDB}{\hfill\ensuremath{\Box}}%

\newcommand{\T}[1]{\boldsymbol{\mathscr{#1}}}

\newcommand{\subfloat}{\subfigure}

\newcommand\blue[1]{\textcolor{blue}{#1}}

%% file: 000abstract.tex
How can we efficiently and accurately analyze an irregular tensor in a dual-way streaming setting where the sizes of two dimensions of the tensor increase over time?
What types of anomalies are there in the dual-way streaming setting?
An irregular tensor is a collection of matrices whose column lengths are the same while their row lengths are different.
In a dual-way streaming setting, both new rows of existing matrices and new matrices arrive over time.
PARAFAC2 decomposition is a crucial tool for analyzing irregular tensors.
Although real-time analysis is necessary in the dual-way streaming,
static PARAFAC2 decomposition methods fail to efficiently work in this setting since they perform PARAFAC2 decomposition for accumulated tensors whenever new data arrive.
Existing streaming PARAFAC2 decomposition methods work in a limited setting and fail to handle new rows of matrices efficiently.

In this paper,
we propose \method, an efficient and accurate PARAFAC2 decomposition method working in the dual-way streaming setting.
When new data are given, \method efficiently performs PARAFAC2 decomposition by carefully dividing the terms related to old and new data and avoiding naive computations involved with old data.
Furthermore, applying a forgetting factor makes \method follow recent movements.
Extensive experiments show that \method achieves up to $14.0\times$ faster speed than existing PARAFAC2 decomposition methods for newly arrived data.
We also provide discoveries for detecting anomalies in real-world datasets, including Subprime Mortgage Crisis and COVID-19. 

%% file: 010intro.tex
How can we efficiently and accurately analyze an irregular tensor in a dual-way streaming setting,
where new rows of existing matrices and new matrices continuously arrive over time?
What types of anomalies are there in the dual-way streaming setting?
Many real-world data can be represented as irregular tensors consisting of matrices where their column sizes are the same, but row sizes are different.
Also, many irregular tensor data are generated in a dual-way streaming setting.
For example, assume that we collect an irregular tensor from a stock market.
The irregular tensor consists of matrices of stocks where rows and columns correspond to listing periods and features (e.g., opening price, trading volume, etc.), respectively.
Feature values of listed stocks are collected continuously, and companies go public.
Then, the sizes of matrices and the number of matrices increase day by day.

There has been much attention on PARAFAC2 decomposition in order to decompose a given irregular tensor into factor matrices, which is used for many applications such as phenotype discovery~\cite{PerrosPWVSTS17,AfsharPPSHS18,ren2020robust,yin2021tedpar}, fault detection~\cite{wise2001application,luo2019sparse}, community detection~\cite{gujral2020spade}, feature analysis~\cite{JangK22}, clustering~\cite{gujral2020c}, and user intent tracking~\cite{sun2016contextual}.
However, many existing PARAFAC2 decomposition methods~\cite{PerrosPWVSTS17,JangK22} have been developed for working in a static setting where the entire tensor is given at once.
Therefore, they fail to efficiently analyze an irregular tensor in a dual-way streaming setting since they need to perform PARAFAC2 decomposition for an accumulated tensor whenever new data arrive.
SPADE~\cite{gujral2020spade} works in a limited streaming setting that handles only new slice matrices;
it does not consider the increase in the size of existing slice matrices.
Furthermore, it focuses only on patterns of the entire duration, and does not
allow us to focus on recent patterns which is necessary in a dual-way streaming setting where new data continuously arrive.

\begin{figure*}
	\subfloat[Overview of \method and its application]{\includegraphics[width=0.99\textwidth]{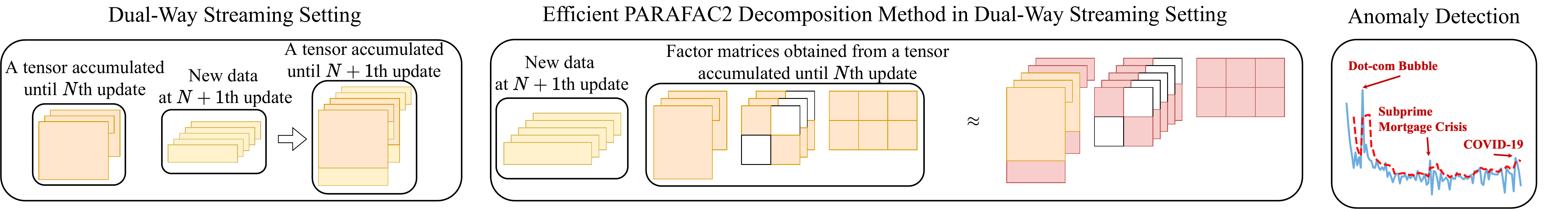}} \\
	\vspace{-3mm}	
	 \subfloat{\includegraphics[width=0.7\textwidth]{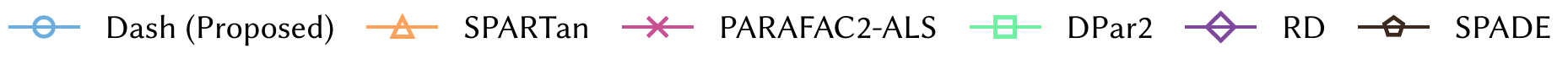}} \\
	\vspace{-3mm}
	 \setcounter{subfigure}{1}		
	\centering	
	 \subfloat[Running time on US Stock dataset]{\includegraphics[width=0.19\textwidth]{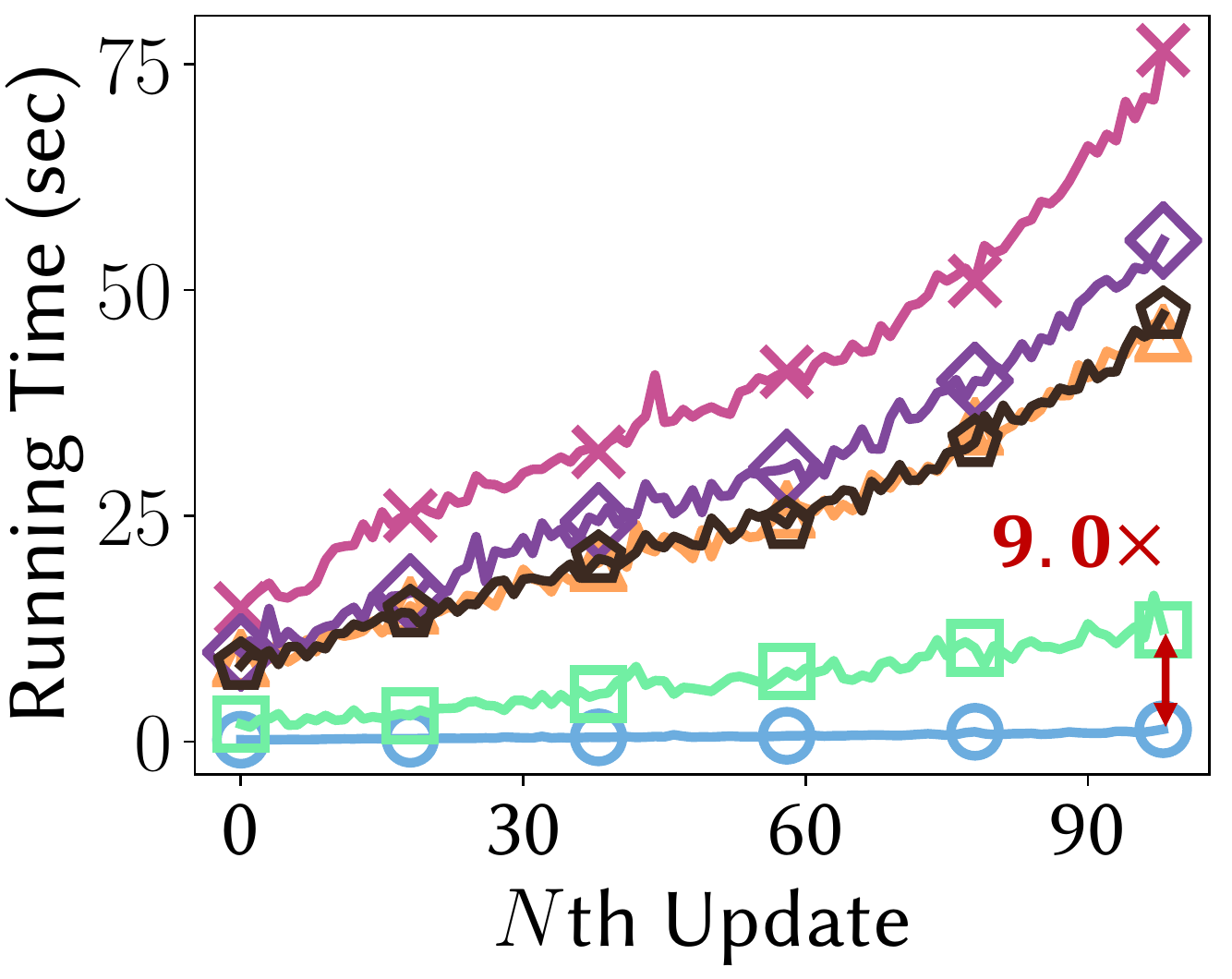}\label{fig:perf_us}}\quad\quad
	 \subfloat[Running time on KR Stock dataset]{\includegraphics[width=0.19\textwidth]{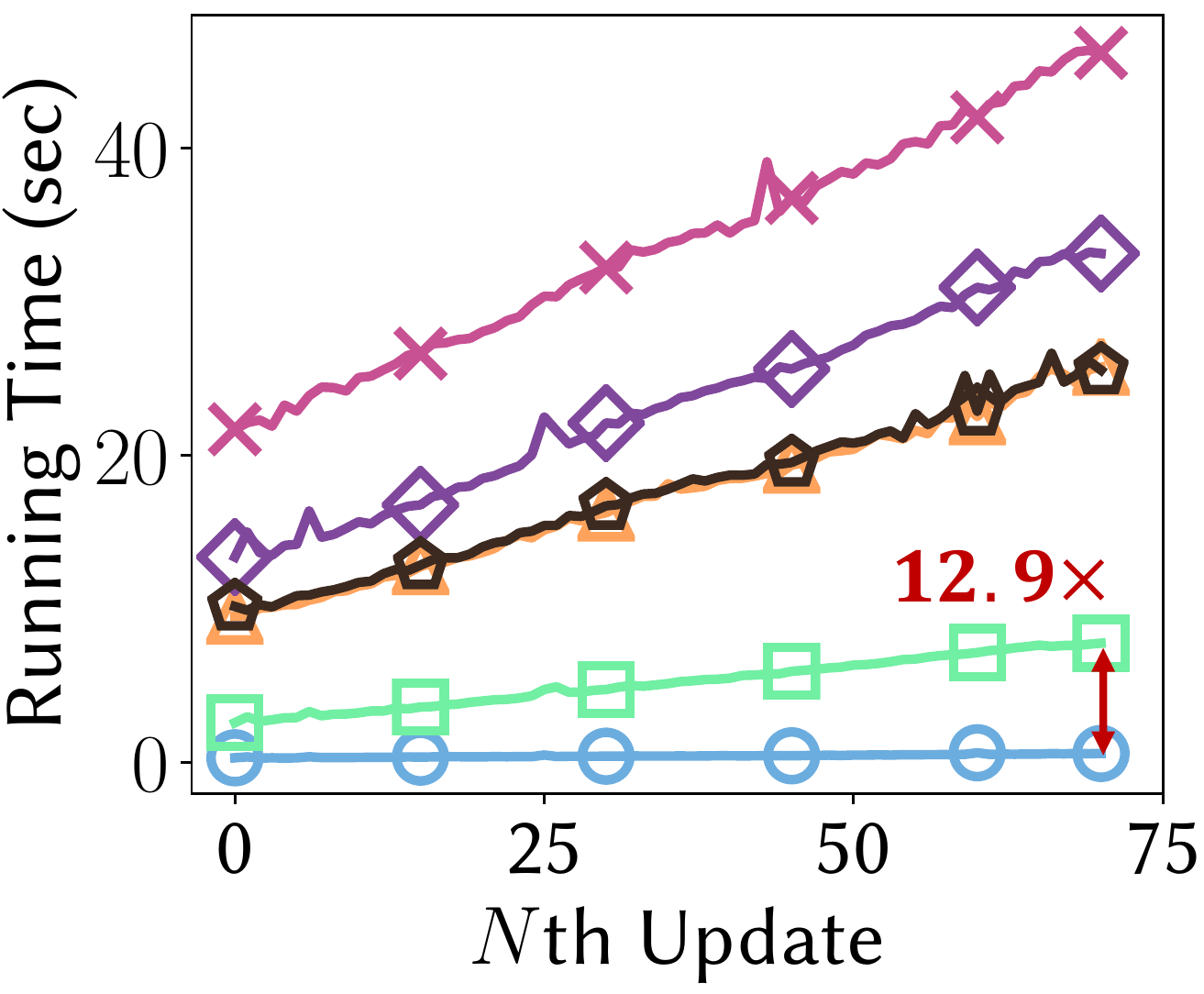}\label{fig:perf_kr}}	 	 	\quad\quad
	 \subfloat[Scalability on US Stock dataset]{\includegraphics[width=0.19\textwidth]{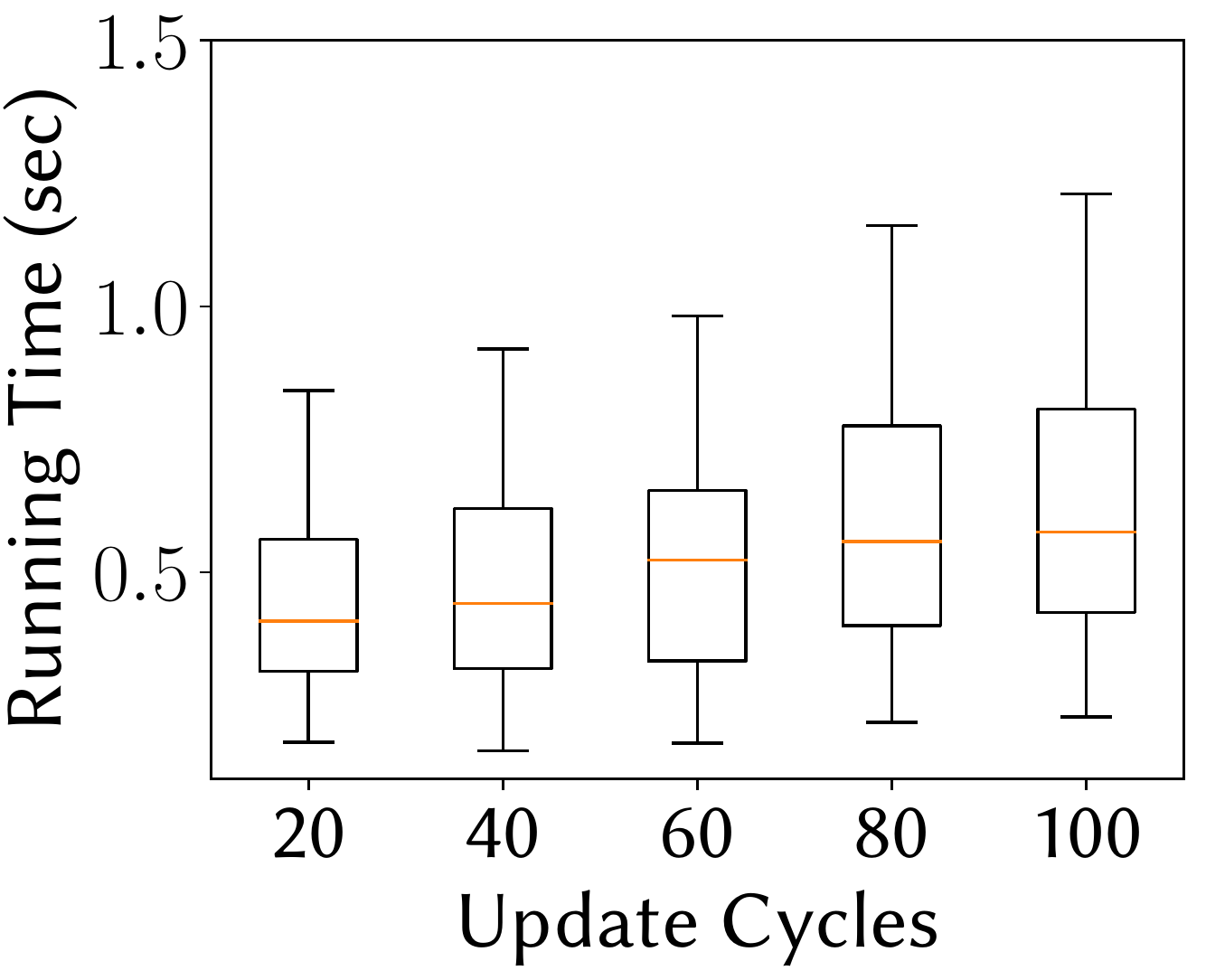}\label{fig:scalability_us}} \quad\quad
	 \subfloat[Scalability on KR Stock dataset]{\includegraphics[width=0.19\textwidth]{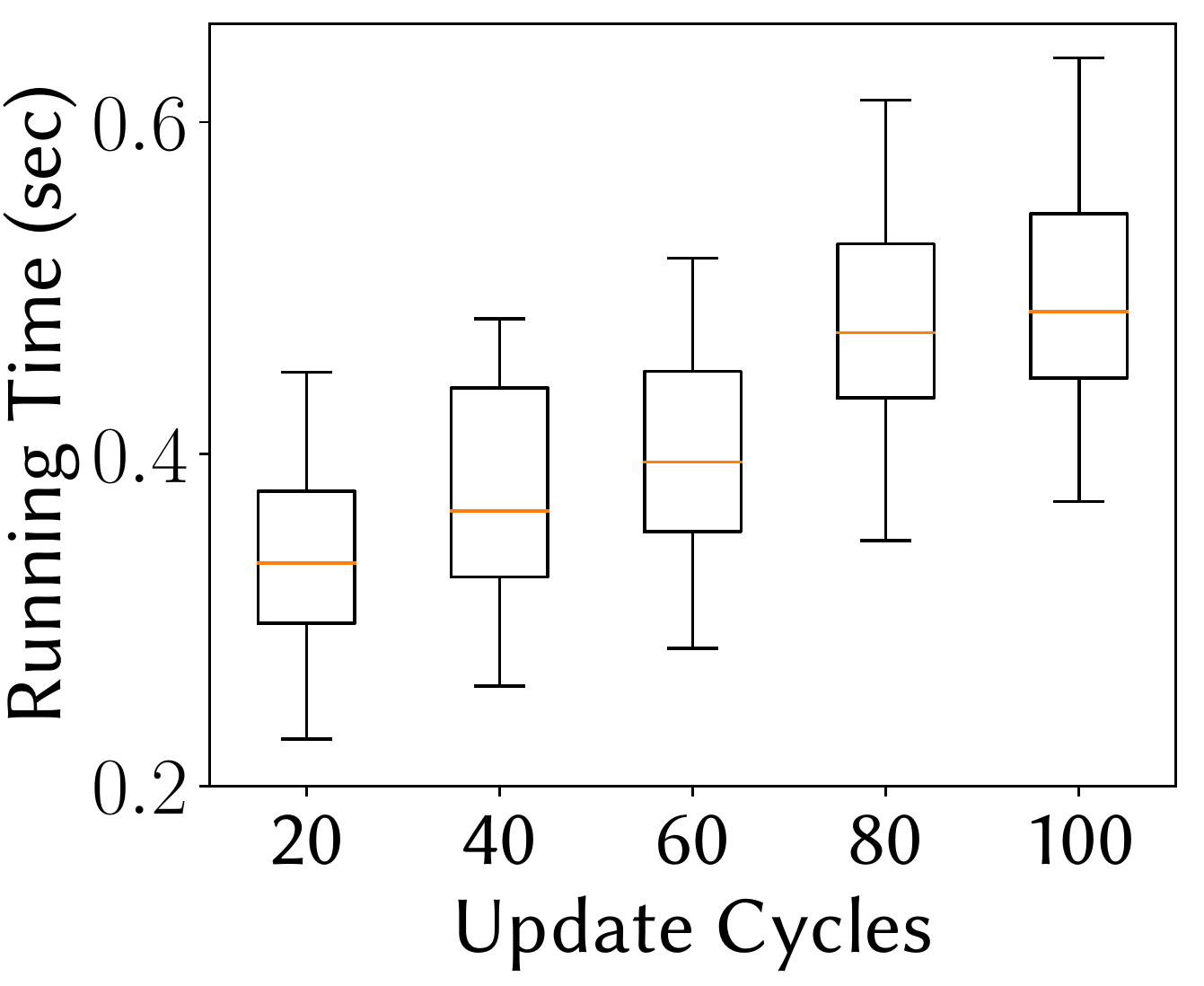}\label{fig:scalability_kr}}	 	 	 	
	 	 \\
	\caption{
	(a) Overview of \method, an efficient PARAFAC2 decomposition method in a dual-way streaming setting where both new rows of existing slice matrices and new slice matrices arrive over time.
	\method allows us to perform real-time anomaly detection in a real-world dataset.
	(b-c) Running time of \method and competitors. For each method, we draw a line connecting points of all updates, and five or six representative points with large markers.
	\method is the fastest method in a dual-way streaming setting with up to $12.9\times$ faster speed than competitors.
	(d-e) Scalability with respect to update cycles.
	We measure running times for several update cycles: $[20, 40, 60, 80, 100]$.
	The length of new rows of existing slice matrices is linearly proportional to an update cycle, and the number of new slice matrices at each update is different depending on an update cycle.
	Running time of \method is linearly proportional to the values of the update cycles.
	}
	\label{fig:contribution}
\end{figure*}

In this paper, we propose \method (\underline{D}ual-w\underline{A}y \underline{S}treaming PARA-FAC2 decomposition met\underline{H}od for irregular tensors), an efficient and accurate method for PARAFAC2 decomposition in a dual-way streaming setting.
In contrast to existing PARAFAC2 decomposition methods,
\method is optimized for the dual-way streaming setting.
\method achieves a high efficiency for dynamically updating factor matrices with an
update rule which requires computational costs linear to the size of a new incoming tensor.
Based on the update rule, \method updates factor matrices by 1) carefully dividing the terms related to old and new data, 2) computing only the terms related to new data, 3) loading old data-related terms, and 4) summing up the terms.
In addition, \method enables to follow recent movements by applying a forgetting factor to our objective function.
We detect various anomalies using \method in a dual-way streaming setting.


Our contributions (see Figure~\ref{fig:contribution}) are summarized as follows:
\begin{itemize*}
	\item \textbf{Problem Formulation.} We formulate a novel problem of handling an irregular tensor in a dual-way streaming setting where both new slice matrices and new rows of existing slice matrices arrive.
	\item \textbf{Method.} We propose \method, an efficient PARAFAC2 decomposition method for an irregular tensor in the dual-way streaming setting.
	\item \textbf{Experiment.} We experimentally show that \method outperforms existing PARAFAC2 methods giving up to $14.0 \times$ faster speed (see Figures~\ref{fig:perf_us},~\ref{fig:perf_kr}, and~\ref{fig:perf}) and achieving lower local errors for a newly arrived tensor (see Table~\ref{tab:local_error}).
	\item \textbf{Discovery.} \method detects economic crisis (i.e., Dot-com Bubble, Subprime Mortgage Crisis, and COVID-19) as anomalies in US Stock dataset (see Figure~\ref{fig:tensorlevel_anomaly}).
\end{itemize*}
%
The code and datasets are available at \\ \textbf{\url{https://github.com/snudatalab/Dash}}.

%% file: 020prelim.tex
In this section, we describe preliminaries for an irregular tensor and PARAFAC2 decomposition,
and then introduce our problem.
Table~\ref{tab:notation} presents notations frequently used in this paper.

\subsection{Preliminaries}
\label{subsec:prelim}

\textbf{Notation.}
We denote vectors and matrices as bold lowercase (e.g., $\mat{x}$) and bold uppercase (e.g., $\mat{X}$), respectively.
$\{\mat{X}_k\}_{k=1}^{K}$ represents a collection of $K$ matrices, which is generally used to define an irregular tensor with $K$ frontal slice matrices.

\begin{table} [t]
	\centering
	\caption{Symbol description.}
	\label{tab:notation}
	\resizebox{0.999\columnwidth}{!}{%
		\begin{tabular}{cl}
			\toprule
			\textbf{Symbol} & \textbf{Description} \\
			\midrule
			$\{ \mat{X}_k \}^{K}_{k=1}$ & irregular tensor of slices $\mat{X}_{k}$ for $k=1,...,K$ \\
			$\mat{X}_{k,old}$ & old rows of an existing slice matrix  \\						
			$\mat{X}_{k,new}$ & \begin{tabular}[c]{@{}l@{}} new rows of $k$th existing slice matrix ($k=1,...,K$)\\ or $k$th new incoming slice matrix ($k=K+1,...K+L$)\end{tabular} \\
			$\mat{X}(i,:)$ & {$i$-th row vector of a matrix $\mat{X}$} \\
			{$\mat{X}(:,j)$} & {$j$-th column vector of a matrix $\mat{X}$} \\
			{$\mat{X}(i,j)$} & {$(i,j)$-th element of a matrix $\mat{X}$} \\
			{$\{\mat{U}_k, \mat{S}_{k}\}_{k=1}^{K}$, $\mat{V}$} & {factor matrices of an irregular tensor} \\
			$\| \cdot \|_F$ & Frobenius norm \\	
			$R$ & target rank \\	
			$\odot$ & Khatri-Rao product \\			
			$*$ & element-wise product \\
			$vec(\cdot)$ & vectorization of a matrix \\
			$| \cdot |$ & absolute value \\
			$I_{k,new}$ & \begin{tabular}[c]{@{}l@{}} length of new rows of $k$th existing slice matrix ($k=1,...,K$) or \\ row length of $k$th new incoming slice matrix ($k=K+1,...K+L$)\end{tabular} \\
			$J$ & the column length of slice matrices \\
			$K$ & the number of slice matrices \\
			$L$ & the number of new slice matrices \\
			\bottomrule
		\end{tabular}
	}
\end{table}

\textbf{PARAFAC2 decomposition.}
PARAFAC2 decomposition has been widely used for analyzing real-world data represented as irregular tensors.
It decomposes a given irregular tensor $\{\mat{X}_{k}\}_{k=1}^{K}$ into factor matrices $\mat{U}_k$, $\mat{S}_k$, and $\mat{V}^T$ (i.e., $\mat{X}_k \approx \mat{U}_k\mat{S}_k\mat{V}^T$) where $\mat{X}_k$ is the $k$th slice matrix of an irregular tensor (see Figure~\ref{fig:parafac2}).
To obtain the factor matrices, ALS (Alternating Least Square) method updates one factor matrix at a time while fixing all other matrices, minimizing the following loss function:
\begin{align}
\label{eq:obj_base}
	\T{L}_{PARAFAC2} =  \sum{}_{k=1}^{K} \| \mat{X}_k - \mat{U}_k\mat{S}_k\mat{V}^T \|_F^2
\end{align}
Many works~\cite{PerrosPWVSTS17,cheng2019efficient,JangK22,gujral2020spade} replace $\mat{U}_k$ with $\mat{Q}_k\mat{H}$ to ensure uniqueness of the solution where $\mat{Q}_k$ is column orthogonal and $\mat{H}$ is constant for all slices.
However, our proposed method does not consider such constraint since it does not require the uniqueness property.

\begin{figure}[t]
	\centering	
	\includegraphics[width=0.35\textwidth]{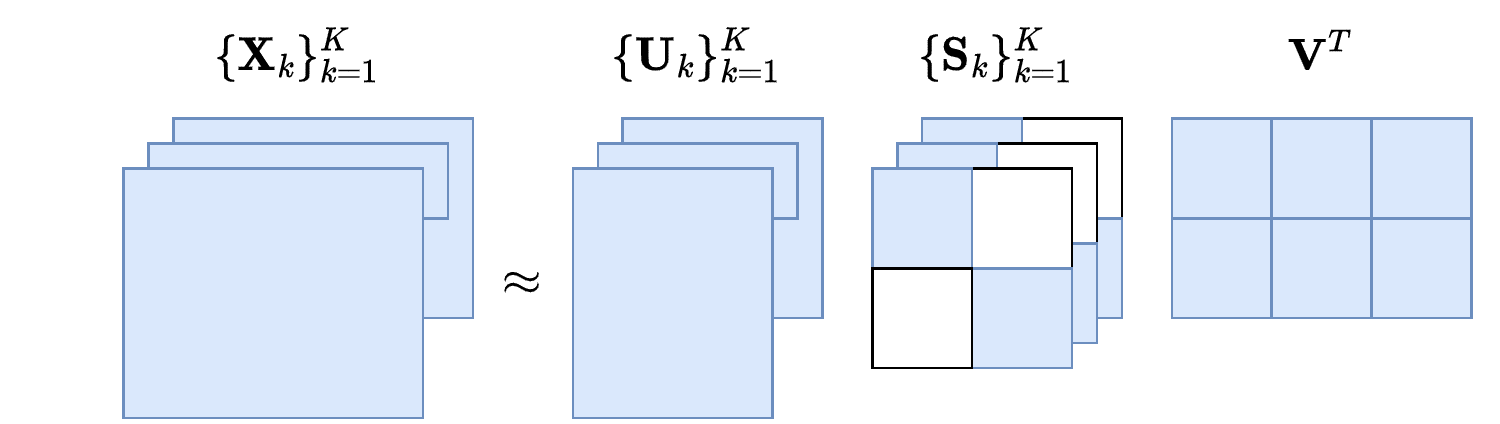}
	\caption{An example of PARAFAC2 decomposition.
	PARAFAC2 decomposes $\mat{X}_k$ into $\mat{U}_k$, $\mat{S}_k$, and $\mat{V}$ for all $k$.
	$\mat{V}$ is the shared factor matrix across all the slice matrices.
	}
	\label{fig:parafac2}
\end{figure}

\subsection{Problem Formulation}

We formulate a problem of analyzing an irregular tensor in a dual-way streaming setting where not only new slice matrices but also new rows of existing slice matrices arrive.

\textbf{Time-related terms.}
Before presenting the problem, we first define the time-related terms that we use throughout this paper. 
Figure~\ref{fig:time_term} provides a visualization of how these terms are used.

\begin{itemize}
	\item \textbf{Time step.} The smallest unit of temporal mode in which new information is collected,
		which is represented as one row of a slice matrix.
		We assume that the row axis of slice matrices is the temporal mode.
	\item \textbf{Update.} To find the new factor matrices with the new information given after certain time steps.
	\item \textbf{Update cycle.} The number of time steps between each update.
	\item \textbf{Updated range of $k$th slice matrix.} The number of new rows in $k$th slice matrix after one update cycle.
		Note that it can be smaller than the update cycle, since not all matrices necessarily collect new information at every time step.
	\item \textbf{Overall duration.} The total number of time steps in data.
\end{itemize}

\textbf{Problem Definition.}
The formal definition is as follows:
\begin{problem}[dual-way streaming Irregular Tensor Problem]

\textbf{Given}
\begin{enumerate}
	\item New slice matrices $\{\mat{X}_{k, new}\}_{k=K+1}^{K+L}$
	\item New rows $\{\mat{X}_{k,new}\}_{k=1}^{K}$ of existing slice matrices
	\item Pre-existing factor matrices $\{\mat{U}_{k,old}\}_{k=1}^{K}$, $\{\mat{S}_{k,old}\}_{k=1}^{K}$, and  $\mat{V}_{old}$ of an accumulated irregular tensor $\{\mat{X}_{k,old}\}_{k=1}^{K}$.
\end{enumerate}

\noindent\textbf{Find} factor matrices $\{\mat{U}_{k}\}_{k=1}^{K+L}$, $\{\mat{S}_{k}\}_{k=1}^{K+L}$, and $\mat{V}$ for the entire irregular tensor $\{\mat{X}_{k}\}_{k=1}^{K+L}$:
\begin{align}
\small
    \mat{X}_k =  \begin{cases}
     \begin{bmatrix}\mat{X}_{k,old} ; \mat{X}_{k,new} \end{bmatrix}       & \text{if } 1 \leq k \leq K \\
  \mat{X}_{k,new}        & otherwise
  \end{cases}
\end{align}
\begin{align}
\small
    \mat{U}_k =  \begin{cases}
     \begin{bmatrix}\mat{U}_{k,old} ; \mat{U}_{k,new} \end{bmatrix}       & \text{if } 1 \leq k \leq K \\
  \mat{U}_{k,new}        & otherwise
  \end{cases}
\end{align}
where $;$ denotes the vertical concatenation of matrices.

\end{problem}

\begin{figure}
	\centering	
	\includegraphics[width=0.49\textwidth]{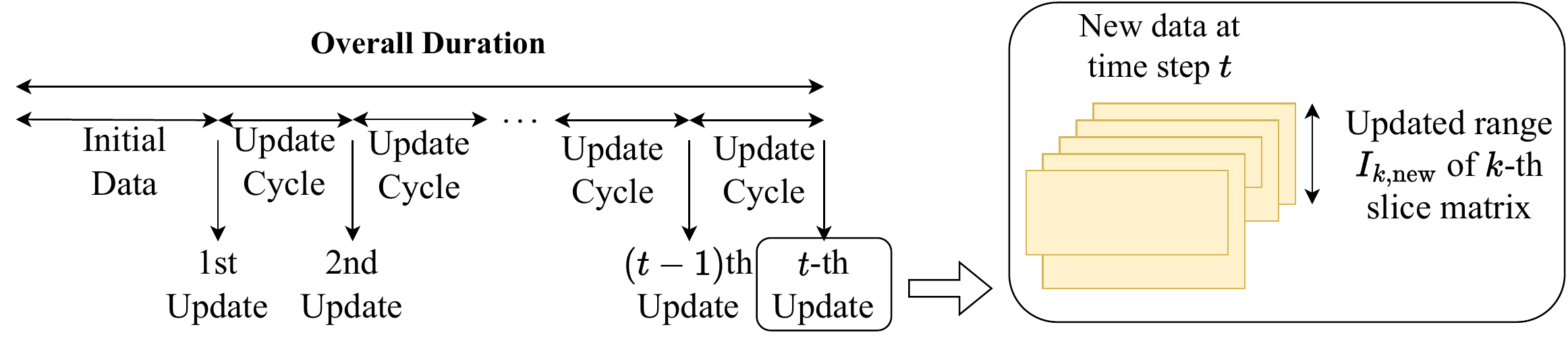}
	\caption{Time-related terms. For example, assume that one day corresponds to one time step, and the overall duration length is $1000$ days. Also assume that the initial tensor consists of 200 days' worth of data and each update cycle is $20$ days. Then, there are $40=(1000-200)/20$ updates. $I_{k,{new}}$ of $k$th slice matrix at each update is less than or equal to $20$.
	}
	\label{fig:time_term}
\end{figure}

%

Many previous works~\cite{PerrosPWVSTS17,cheng2019efficient,JangK22,gujral2020spade} efficiently perform PARAFAC2 decomposition for analyzing irregular tensors.
However, there is no method tailored for the dual-way streaming setting.
The static PARAFAC2 decomposition method is impractical since they use an accumulated tensor instead of pre-existing factor matrices.
Therefore, to achieve high efficiency, a dual-way streaming PARAFAC2 decomposition method needs to fully employ pre-existing factor matrices obtained at the previous update instead of the accumulated data.
Although SPADE~\cite{gujral2020spade} handles newly arrived slice matrices, it fails to handle new rows of existing slice matrices.
This limitation is fatal to the efficiency of updating factor matrices when new rows of existing slice matrices continuously arrive.
This is because this method needs to update factor matrices using the accumulated tensor with static PARAFAC2 decomposition methods.
Therefore, we need to develop an efficient method working in a dual-way setting where new rows of existing slice matrices and new slice matrices simultaneously arrive.

%% file: 030proposed.tex
In this paper, we propose \method (\underline{D}ual-w\underline{A}y \underline{S}treaming PARA-FAC2 decomposition met\underline{H}od for irregular tensors) which efficiently performs PARAFAC2 decomposition in a dual-way streaming setting.
There are main challenges to be addressed:

\begin{itemize*}
	\item[C1.] \textbf{Maximizing efficiency in a dual-way streaming setting.} Previous methods fail to work efficiently in a dual-way streaming setting since there are computations with old data. How can we efficiently update factor matrices in a dual-way streaming setting?
	\item[C2.] \textbf{Focusing on recent movements.} Since data are continuously accumulated over time, we need to follow recent behaviors and discard old data gradually. How can we compute factor matrices of PARAFAC2 decomposition with more importance to newly coming data?  
	\item[C3.] \textbf{Detecting anomalies in a dual-way streaming setting.} 
What types of anomalies are there in a dual-way streaming setting? How can we detect anomalies?
\end{itemize*}

We propose the following main ideas to address the challenges.
\begin{itemize*}
	\item[I1.] \textbf{Divide and compute.} \method achieves high efficiency by 1) dividing the terms related to old and new data, 2) computing only the term of new data, 3) loading the term related to old data without computation, and 4) summing up the loaded term and the computed term.
	\item[I2.] \textbf{Exploit a forgetting factor.} A forgetting factor makes factor matrices fit more to new incoming data while gradually excluding old data in updating factor matrices.
	\item[I3.] \textbf{Detect Slice-level and Tensor-level anomalies.} We detect two types of anomalies by measuring reconstruction errors with factor matrices updated by \method whenever new data arrive.
\end{itemize*}

  \begin{algorithm} [t]
	\caption{\method}\label{alg:odin_algorithm}
	\begin{algorithmic} [1]
		\small
		\algsetup{linenosize=\small}
		\SetNoFillComment
		\renewcommand{\algorithmicrequire}{\textbf{Input:}}
		\renewcommand{\algorithmicensure}{\textbf{Output:}}
		    \REQUIRE a new incoming tensor $\{\mat{X}_{k,new}\}_{k=1}^{K+L}$, pre-existing factor matrices $\{\mat{U}_{k,old}\}_{k=1}^{K}$, $\{\mat{S}_{k,old}\}_{k=1}^{K}$, and $\mat{V}_{old}$, and helper matrices $\{\mat{c}_{k,old}\}_{k=1}^{K}$, $\{\mat{D}_{k,old}\}_{k=1}^{K}$, $\mat{F}_{old}$, and $\mat{G}_{old}$
		    \ENSURE updated factor matrices $\{\mat{U}_{k,new}\}_{k=1}^{K+L}$, $\{\mat{S}_{k}\}_{k=1}^{K+L}$, and $\mat{V}$, and helper matrices $\{\mat{c}_{k,new}\}_{k=1}^{K+L}$, $\{\mat{D}_{k,new}\}_{k=1}^{K+L}$, $\mat{F}_{new}$, and $\mat{G}_{new}$
		\FOR {$k=1,...,K+L$}
			\STATE update $\mat{U}_{k,new}$ using Eq.~\eqref{eq:updateU_lemma} \\
		\ENDFOR \\
		\FOR {$k=1,...,K+L$}
			\STATE obtain $\mat{c}_{k,new}$ and $\mat{D}_{k,new}$ as in Eq.~\eqref{eq:obtain_CandC} and ~\eqref{eq:obtain_CandD} \\
			\STATE update $\mat{S}_{k} \leftarrow \mat{W}(k,:)$ using $\mat{c}_{k,new}$ and $\mat{D}_{k,new}$ as in Eq.~\eqref{eq:updateS_lemma} \\
	\tcc{$\mat{S}_{k}$, $\mat{c}_{k,new}$, and $\mat{D}_{k,new}$ are used as \newline
	$\mat{S}_{k,old}$, $\mat{c}_{k,old}$, and $\mat{D}_{k,old}$ at the next update.} 					
		\ENDFOR			\label{alg:line:update_s} \\							
		\STATE obtain $\mat{F}_{new}$ and $\mat{G}_{new}$ as in Eq.~\eqref{eq:obtain_FandF} and~\eqref{eq:obtain_FandG}  \\
		\STATE update $\mat{V}$ using $\mat{F}_{new}$ and $\mat{G}_{new}$ as in Eq.~\eqref{eq:updateV_lemma} \\		
	\tcc{$\mat{V}$, $\mat{F}_{new}$, and $\mat{G}_{new}$ are used as \newline
	$\mat{V}_{old}$, $\mat{F}_{old}$, and $\mat{G}_{old}$ at the next update.} 			
	\end{algorithmic}
\end{algorithm}

We first find the factor matrices of an initial tensor described in Figure~\ref{fig:time_term}.
Then,
we update factor matrices $\{\mat{U}_{k,new}\}_{k=1}^{K+L}$, $\{\mat{S}_{k}\}_{k=1}^{K+L}$, and $\mat{V}$ in order whenever a new incoming tensor arrives.
Algorithm~\ref{alg:odin_algorithm} presents the update procedure of \method for a new incoming tensor.
Our update rule, which avoids the computations involved with an accumulated irregular tensor,
allows us to efficiently update them for a new incoming tensor.

\subsection{Objective Function with Forgetting Factor}
\label{subsec:objective}
We start from designing a loss function for a dual-way streaming setting.
The loss function in Eq.~\eqref{eq:obj_base} is designed for static PARAFAC2 decomposition. 
It is not appropriate for efficient updates of factor matrices in a dual-way streaming setting since it leads to the computations involved with an accumulated tensor.
Therefore, we need to reformulate the loss function tailored for a dual-way streaming setting.
In addition, a reformulated loss function needs to capture recent behaviors while discarding old data gradually.
Our main ideas are 1) to distinguish old data and newly arrived data within the loss function, and 2) to add a forgetting factor to a loss function.

We reformulate the loss function tailored for a dual-way streaming setting.
When a new tensor is given, the loss function is expressed as follows:
\begin{align}
\label{eq:division_loss}
\begin{split}
	\T{L} & = \sum_{k=1}^{K}{\| \begin{bmatrix} \mat{X}_{k,old} \\ \mat{X}_{k,new} \end{bmatrix} - \begin{bmatrix} \mat{U}_{k,old} \\ \mat{U}_{k,new} \end{bmatrix} \mat{S}_k\mat{V}^T \|_F^2} \\
	& + \sum_{k=K+1}^{K+L} {\| \mat{X}_{k,new} - \mat{U}_{k,new}\mat{S}_k\mat{V}^T \|_F^2}
\end{split}
\end{align}
We divide $\mat{X}_{k}$ into $\mat{X}_{k,old}$ and $\mat{X}_{k,new}$, and divide $\mat{U}_{k}$ into $\mat{U}_{k,old}$ and $\mat{U}_{k,new}$.
By adding a forgetting factor to Eq.~\eqref{eq:division_loss}, our reformulated loss function is as follows:
\begin{align}
\small
\label{eq:new_loss}
\begin{split}
	\T{L} & = \lambda\sum_{k=1}^{K}{\| \mat{X}_{k,old} - \mat{U}_{k,old}\mat{S}_k\mat{V}^T \|_F^2}
	 + \sum_{k=1}^{K+L} {\| \mat{X}_{k,new} - \mat{U}_{k,new}\mat{S}_k\mat{V}^T \|_F^2	}
\end{split}
\end{align}
where $\lambda$ $(0 < \lambda \leq 1)$ is the hyperparameter of a forgetting factor.
The forgetting factor used in research~\cite{vahidi2005recursive,PapadimitriouSF05} makes \method fit factor matrices to newly arrived data while discarding old data.
The smaller $\lambda$ is, the more \method fits factor matrices to recently arrived data.
Our loss function clearly distinguishes old and new data.

\subsection{Dual-way Streaming Update}
The objective is to efficiently update factor matrices in a dual-way streaming setting.
Existing methods~\cite{PerrosPWVSTS17,cheng2019efficient,JangK22,gujral2020spade} update factor matrices using the entire tensor, so they are very inefficient for new incoming data.
Although SPADE~\cite{gujral2020spade} can handle new incoming slice matrices, it fails to efficiently update factor matrices with respect to new additions to existing slice matrices.
Therefore, we need to develop an update rule that efficiently updates factor matrices for both new slice matrices and the additions to existing slice matrices.

For the goal, we formulate a divide-and-compute strategy for efficiency.
This strategy excludes computations involved with an existing tensor.
Specifically, we 1) divide terms related to old and new data, 2) compute only the term for new data, 3) load the term related to old data without computation, and 4) sum up the loaded term and the computed term.
Dividing the terms is addressed by the reformulated loss function (Eq.~\eqref{eq:new_loss}).
Thanks to this strategy, we avoid exponential growth in computation over time.

\subsubsection{Initialization}
Before starting an update procedure, we find the factor matrices $\mat{U}_{k,old}$, $\mat{S}_k$, and $\mat{V}$ of an initial tensor $\{\mat{X}_{k, initial}\}_{k=1}^{K}$ and helper matrices $\mat{c}_{k,old} \in \mathbb{R}^{R}$, $\mat{D}_{k,old}\in \mathbb{R}^{R\times R}$, $\mat{F}_{old} \in \mathbb{R}^{J\times R}$, and $\mat{G}_{old}\in \mathbb{R}^{R\times R}$ for supporting efficient updates.
First, we perform PARAFAC2 decomposition for the initial tensor; we use SPARTan~\cite{PerrosPWVSTS17} for the initialization in this paper.
Then, we obtain helper matrices $\mat{c}_{k,old}$, $\mat{D}_{k,old}$, $\mat{F}_{old}$, and $\mat{G}_{old}$ as follows:
\begin{align}
	& \mat{c}_{k,old}^T \leftarrow vec(\mat{X}_{k,initial})^T(\mat{V} \odot \mat{U}_{k,old}) \label{eq:updateC_init} \\
	& \mat{D}_{k,old} \leftarrow \mat{U}_{k,old}^T\mat{U}_{k,old} \label{eq:updateD_init} \\
	& \mat{F}_{old} \leftarrow  \sum_{k} {\mat{X}_{k,initial}^T\mat{U}_{k,old}\mat{S}_k} \label{eq:updateF_init} \\	
	& \mat{G}_{old} \leftarrow  \sum_{k}{\mat{S}_k\mat{U}_{k,old}^T\mat{U}_{k,old}\mat{S}_k} \label{eq:updateG_init}
\end{align}
Note that $\mat{c}_{k,old}$ and $\mat{D}_{k,old}$ support an efficient update for the factor matrix $\mat{S}_k$
while $\mat{F}_{old}$ and $\mat{G}_{old}$ help update the factor matrix $\mat{V}$ efficiently.

\subsubsection{Updating a factor matrix $\mat{U}_k$}
\method finds new row factors $\mat{U}_{k,new}$ of the factor matrix $\mat{U}_k$ when a new slice matrix or new rows of an existing slice matrix is given.
To efficiently find $\mat{U}_{k,new}$, we focus on the following term from Eq.~\eqref{eq:new_loss}.
\begin{align}
\label{eq:new_lossU}
	\T{L}_{\mat{U}_{k,new}} = \| \mat{X}_{k,new} - \mat{U}_{k,new}\mat{S}_k\mat{V}^T \|_F^2	
\end{align}
Note that the term $\| \mat{X}_{k,old} - \mat{U}_{k,old}\mat{S}_k\mat{V}^T \|_F^2$ from Eq.~\eqref{eq:new_loss}
is unnecessary in updating $\mat{U}_{k,new}$.
Then, we set $\frac{\partial \T{L}_{\mat{U}_{k,new}}}{\partial \mat{U}_{k,new}} = 0$ and derive the update rule for $\mat{U}_{k,new}$ as follows:
\begin{lemma}
\label{lemma:update_u}
New row factors $\mat{U}_{k,new}$ of the factor matrix $\mat{U}_k$ are obtained with the following update rule:
\vspace{-1mm}
\begin{align}
\label{eq:updateU_lemma}
\begin{split}
		&\mat{U}_{k,new} \leftarrow \mat{X}_{k,new}\mat{V}\mat{S}_k\left(\mat{S}_k\mat{V}^T\mat{V}\mat{S}_k\right)^{-1}
\end{split}
\end{align}
where $\mat{X}_{k,new}$ corresponds to newly arrived rows of the $k$th slice matrix $\mat{X}_k$.
\QEDB
\end{lemma}
\begin{proof}
\vspace{-1mm}
	The proof of Lemma~\ref{lemma:update_u} is described in Appendix~\ref{subsec:proof_lemmaU}.
\end{proof}
Since we fully exclude the computation involved with $\mat{X}_{k,old}$ and $\mat{U}_{k,old}$,
\method efficiently obtains $\mat{U}_{k,new}$ with the computational time proportional to the size of newly arrived data. 
Note that $\mat{U}_{k,old}$ is not updated.

\subsubsection{Updating a factor matrix $\mat{S}_k$}
We update factor matrices $\mat{S}_k$ when new data including new slice matrices and new rows of existing slice matrices are given.
We first re-express $\mat{S}_{k}$ for $k=1,...,K+L$ as $\mat{W}\in \mathbb{R}^{(K+L)\times R}$ where the $k$th row of $\mat{W}$ corresponds to the diagonal elements of $\mat{S}_{k}$.
Then, we set $\frac{\partial \T{L}}{\partial \mat{W}(k,:)} = 0$ and update the factor matrix $\mat{W}(k,:)$ for $k=1,...,K+L$ as follows:
\begin{lemma}
\label{lemma:update_s}
The factor matrix $\mat{W}(k,:)$ is obtained with the following update rule:
\vspace{-1mm}
\begin{align}
\label{eq:updateS_lemma}
		&\mat{W}(k,:) \leftarrow \mat{c}_{k,new}^T \times \left( \mat{V}^T\mat{V} * \mat{D}_{k,new}\right)^{-1}
\end{align}
where
\begin{align}
\small
\begin{split}
\label{eq:updateC}
		& \mat{c}_{k,new}^T
		= \lambda \cdot vec(\mat{X}_{k,old})^T(\mat{V} \odot \mat{U}_{k,old}) + vec(\mat{X}_{k,new})^T(\mat{V} \odot \mat{U}_{k,new})
		\end{split} \\
		\begin{split}		
\label{eq:updateD}		
		& \mat{D}_{k,new} = \lambda \cdot \mat{U}_{k,old}^T\mat{U}_{k,old} + \mat{U}_{k,new}^T\mat{U}_{k,new}
		\end{split}
		\small
\end{align}
$\mat{X}_{k,new}$ corresponds to newly arrived rows of the $k$th slice matrix.
\QEDB
\end{lemma}
\begin{proof}
	\vspace{-1mm}
	The proof of Lemma~\ref{lemma:update_s} is described in Appendix~\ref{subsec:proof_lemmaS}.
\end{proof}
$\mat{W}(k,:)$ is assigned to the diagonal elements of $\mat{S}_{k}$.
In Eq.~\eqref{eq:updateC} and~\eqref{eq:updateD},
computing the terms involved with $\mat{X}_{k,old}$ and $\mat{U}_{k,old}$ requires heavy computational costs proportional to the size of the accumulated tensor.
To achieve high efficiency, at the $N$th update,
we avoid directly computing the terms with $\mat{X}_{k,old}$ and $\mat{U}_{k,old}$ by exploiting helper matrices $\mat{c}_{k,old}$ and $\mat{D}_{k,old}$ obtained at the $(N-1)$th update:
\begin{align}
	 & \mat{c}_{k,new}^T
	 \leftarrow \lambda \cdot \mat{c}_{k,old}^T +  vec(\mat{X}_{k,new})^T(\mat{V} \odot \mat{U}_{k,new}) 	\label{eq:obtain_CandC}
\\
	& \mat{D}_{k,new} \leftarrow
	\lambda \cdot \mat{D}_{k,old} + \mat{U}_{k,new}^T\mat{U}_{k,new} \label{eq:obtain_CandD}
\end{align}
where $\mat{c}_{k,old}$ and $\mat{D}_{k,old}$ at the $N$th update are equal to $\mat{c}_{k,new}$ and $\mat{D}_{k,new}$ obtained at the $(N-1)$th update, respectively.
For the first update, we use $\mat{c}_{k,old}$ and $\mat{D}_{k,old}$ obtained at the initialization (Eq.~\eqref{eq:updateC_init} and~\eqref{eq:updateD_init}).

To update the new factor matrix $\mat{S}_k$ at the $N$th update, we 1) compute $vec(\mat{X}_{k,new})^T$ $(\mat{V} \odot \mat{U}_{k,new})$ and $\mat{U}_{k,new}^T\mat{U}_{k,new}$, 2) load $\mat{c}_{k,old}$ and $\mat{D}_{k,old}$ computed at the $(N-1)$th update, and 3) complete the update by performing the remaining operations (i.e., summation, element-wise product, and matrix multiplication).
Note that we use $\mat{c}_{k,new}$ and $\mat{D}_{k,new}$ as $\mat{c}_{k,old}$ and $\mat{D}_{k,old}$, respectively, at the $(N+1)$th update.

\subsubsection{Updating a factor matrix $\mat{V}$}
We update the factor matrix $\mat{V}$ when new data including new slice matrices and new rows of existing slice matrices are given.
From Eq.~\eqref{eq:new_loss}, we set $\frac{\partial \T{L}}{\partial \mat{V}} = 0$ and derive the update rule for $\mat{V}$ as follows:
\begin{lemma}
\label{lemma:update_v}
The factor matrix $\mat{V}$ is updated with the following update rule:
\vspace{-1mm}
\begin{align}
\label{eq:updateV_lemma}
\begin{split}
		&
		\mat{V} \leftarrow \mat{F}_{new}\mat{G}_{new}^{-1}\end{split}
\end{align}
where
\begin{align}
\small
\begin{split}
\label{eq:updateF}
		& \mat{F}_{new}
		= \left(\lambda\sum_{k=1}^{K}{\mat{X}_{k,old}^T\mat{U}_{k,old}\mat{S}_k}+\sum_{k=1}^{K+L}{\mat{X}_{k,new}^T\mat{U}_{k,new}\mat{S}_k}\right)
		\end{split} \\
		\begin{split}		
\label{eq:updateG}		
		& \mat{G}_{new} = \left(\lambda\sum_{k=1}^{K}{\mat{S}_k\mat{U}_{k,old}^T\mat{U}_{k,old}\mat{S}_k + \sum_{k=1}^{K+L}{\mat{S}_k\mat{U}_{k,new}^T\mat{U}_{k,new}\mat{S}_k}}\right)
		\end{split}
		\small
\end{align}
$\mat{X}_{k,new}$ corresponds to newly arrived rows of the $k$th slice matrix.
\QEDB
\end{lemma}
\begin{proof}
\vspace{-1mm}
	The proof of Lemma~\ref{lemma:update_v} is described in Appendix~\ref{subsec:proof_lemmaV}.
\end{proof}
In Eq.~\eqref{eq:updateF} and~\eqref{eq:updateG}, a naive approach is to compute the terms involved with $\mat{X}_{k,old}$ and $\mat{U}_{k,old}$.
However, the computations require heavy computational costs proportional to the size of the accumulated tensor.
To achieve high efficiency, at the $N$th update, we avoid the direct computations with $\mat{X}_{k,old}$ and $\mat{U}_{k,old}$ by using $\mat{F}_{old}$ and $\mat{G}_{old}$ obtained at the $N-1$th update:
\begin{align}
\label{eq:obtain_FandF}
	 &\mat{F}_{new} \leftarrow  \lambda \cdot \mat{F}_{old} + \sum\nolimits_{k=1}^{K+L}{(\mat{X}_{k,new}^T\mat{U}_{k,new}\mat{S}_k)} \\
	 &\mat{G}_{new} \leftarrow  \lambda \cdot \mat{G}_{old} + \sum\nolimits_{k=1}^{K+L}{(\mat{S}_k\mat{U}_{k,new}^T\mat{U}_{k,new}\mat{S}_k)} \label{eq:obtain_FandG}
\end{align}
where $\mat{F}_{old}$ and $\mat{G}_{old}$ at the $N$th update are equal to $\mat{F}_{new}$ and $\mat{G}_{new}$ obtained at the $N-1$th update, respectively.
For the first update, we use $\mat{F}_{old}$ and $\mat{G}_{old}$ obtained at the initialization (Eq.~\eqref{eq:updateF_init} and~\eqref{eq:updateG_init}).

At the $N$th update, we efficiently update the factor matrix $\mat{V}$ by 1) computing $\mat{X}_{k,new}^T$ $\mat{U}_{k,new}\mat{S}_k$ and $\mat{S}_k\mat{U}_{k,new}^T$ $\mat{U}_{k,new}\mat{S}_k$, 2) loading $\mat{F}_{old}$ and $\mat{G}_{old}$ computed at the $(N-1)$th update, and 3) completing the update by performing the remaining operations (i.e., summation and matrix multiplication).
Note that we use $\mat{F}_{new}$ and $\mat{G}_{new}$ as $\mat{F}_{old}$ and $\mat{G}_{old}$, respectively, at the $(N+1)$th update.



\subsection{Time Complexity}
\label{subsec:time_complexity}
We provide the time complexity of \method for newly arrived data.

\begin{theorem}
\label{theorem:time_complexity}
Given a new incoming tensor of the size $J\sum_{k=1}^{K+L}{I_{k,new}}$, the time complexity of \method is $\T{O}(JR\sum_{k=1}^{K+L}{I_{k,new}})$.
\end{theorem}
\begin{proof}
\vspace{-1mm}
	The proof is described in Appendix~\ref{subsec:proof_time_complexity}.
\end{proof}

\begin{table*}[t]
	\caption{Description of real-world tensor datasets.
	}
	\centering
	\label{tab:Description}
		\resizebox{0.9\textwidth}{!}{
	\begin{tabular}{lrrrrrrrc}
		\toprule
		& \multicolumn{3}{c}{\textbf{Total Tensor Size}} & \multicolumn{3}{c}{\textbf{Initial Tensor Size}} & \\
		\textbf{Dataset} & \textbf{Max Dim. $I_k$} & \textbf{Dim. $J$} & \textbf{Dim. $K$} & \textbf{Max Dim. $I_k$} & \textbf{Dim. $J$} & \textbf{Dim. $K$} & \textbf{Update Cycle} & \textbf{Newly Arrived Data} \\
		\midrule
		US Stock$^{\ref{footnote:us}}$~\cite{JangK22} &  $7,397$ & $85$ & $4,742$ &  $1,458$ & $85$ & $994$ & 60 days & \multirow{2}{*}{\begin{tabular}[c]{@{}c@{}}New rows of existing slice matrices \\ and new slice matrices\end{tabular}} \\
		KR Stock$^{\ref{footnote:kr}}$~\cite{JangK21} &  $5,268$ & $85$ & $3,620$ &  $995$ & $85$ & $2,092$ & 60 days  \\
		\midrule
		JPN Stock$^{\ref{footnote:jpn}}$ &  $2,204$ & $85$ & $215$ &  $424$ & $85$ & $215$  & 20 days & \multirow{4}{*}{\begin{tabular}[c]{@{}c@{}}New rows of \\ existing slice matrices\end{tabular}} \\
		CHN Stock$^{\ref{footnote:jpn}}$ &  $2,431$ & $85$ & $219$ &  $471$ & $85$ & $219$  & 20 days \\						
		VicRoads$^{\ref{footnote:vicroads}}$~\cite{schimbinschi2015traffic} &  $1,084$ & $96$ & $2,033$ &  $184$ & $96$ & $2,033$  & 20 days \\					
		PEMS$^{\ref{footnote:pems}}$ &  $440$ & $144$ & $963$ &  $80$ & $144$ & $963$  & 20 days \\				
		\bottomrule
	\end{tabular}}
\end{table*}

\textbf{Comparison between \method and existing methods.}
There are four settings in a dual-way streaming.
\begin{enumerate}
	\item No data is available.
	\item Only new rows of existing slice matrices arrive.
	\item Only new slice matrices arrive.
	\item Both new slice matrices and new rows of existing slice 	matrices arrive.
\end{enumerate}
We do not have to consider the first setting since there is no update.
Second, if only new rows of existing matrices arrive over time while new slice matrices do not,
\method has the update time equal to $\T{O}(JR\sum_{k=1}^{K}{I_{k,new}})$.
If the update cycle is constant, the update time of \method does not increase since the size of newly arrived data is not changed for each update.
In contrast, existing static methods and \spade require running time proportional to the size of the accumulated tensor.
In the third setting, \method updates factor matrices with the running time equal to $\T{O}(JR\sum_{k=K+1}^{K+L}{I_{k,new}})$.
\spade also requires the running time proportional to the size of new slice matrices while existing static methods still require the running time proportional to the total size of the accumulated tensor.
In the last setting where both new rows of existing matrices and new slice matrices arrive over time, \method takes $\T{O}(JR\sum_{k=1}^{K+L}{I_{k,new}})$ time which is linearly proportional to the size of new data.
\spade efficiently updates factor matrices for new slice matrices, but it requires a heavy computational cost proportional to the size of accumulated existing slice matrices.
Existing static methods still require the total size of the accumulated tensor.

\subsection{Anomaly Detection with \method}
\label{subsec:application}

What anomalies exist in a dual-way streaming setting?
How can we find them using \method?
As new rows of slice matrices and new slice matrices arrive over time, various anomalies are hidden in the new ones.
Although factor matrices are updated using the new incoming data, anomalies are still very different from predictions by the updated factor matrices.
\method efficiently discovers two types of anomalies, slice-level anomaly and tensor-level anomaly, by measuring reconstruction errors with updated factor matrices.


%

\textbf{Slice-level anomaly detection.}
We discover slice-level anomalies by measuring reconstruction errors of slice matrices and
finding an update that provokes a large error.
We compute reconstruction errors of slice matrices based on the following definition.
\begin{definition}[Slice-level reconstruction error]
\label{def:slice_level}
	When new rows $\mat{X}_{k,new}$ of a $k$th existing slice matrix ($k=1,..,K$) or a new $k$th slice matrix $\mat{X}_{k,new}$ ($k=K+1,..,K+L$) is given, a slice-level reconstruction error is as follows:
\begin{align}
\small
	se\left(\mat{X}_{k,new}\right) = \frac{1}{I_{k,new} J}{\sum_{i,j}{{|\mat{X}_{k,new}(i,j)-\hat{\mat{X}}_{k,new}(i,j)|}}}
\end{align}
where $\hat{\mat{X}}_{k,new}$ is the reconstructed matrix by $\mat{U}_{k,new}\mat{S}_{k}\mat{V}^T$.
\end{definition}
We identify an anomaly for each slice matrix in each update,
by spotting an error higher than a threshold.

\textbf{Tensor-level anomaly detection.}
We detect tensor-level anomalies by measuring a reconstruction error for newly arrived data.
There are newly arrived data that fail to be reconstructed well.
When new data arrive, we obtain a reconstruction error by comparing newly arrived data and the reconstruction of updated factor matrices.
Then, we compare it with a threshold to identify an anomaly.
We define the tensor-level reconstruction error as follows.
\begin{definition}[Tensor-level reconstruction error]
\label{def:tensor_level}
	When new rows $\mat{X}_{k,new}$ of existing slice matrices for $k=1,...,K$ and new slice matrices $\mat{X}_{k,new}$ $(k=K+1,...,K+L)$ are given, tensor-level reconstruction error is as follows:
\begin{align}
\small
\begin{split}
	  te\left(\{\mat{X}_{k,new}\}_{k=1}^{K+L}\right) =
	  \frac{1}{K+L}\sum_{k=1}^{K+L}\frac{1}{I_{k,new}J}{\sum_{i,j}{|\mat{X}_{k,new}(i,j)-\hat{\mat{X}}_{k,new}(i,j)|}}
\end{split}
\end{align}
where $\hat{\mat{X}}_{k,new}$ is reconstructed by $\mat{U}_{k,new}\mat{S}_{k}\mat{V}^T$.
\end{definition}
At each update, we identify an anomaly where a reconstruction error is higher than a threshold.


%% file: 040experim.tex
In this section, we provide experimental results to answer the following questions:
\begin{itemize*}
	\item[Q1.] 	\textbf{Performance (Section~\ref{subsec:performance}).} How quickly and accurately does \method update factor matrices when new slice matrices and new data of existing slice matrices arrive?
	\item[Q2.] \textbf{Scalability (Section~\ref{subsec:scalability}).} How does the size of newly arrived data affect the running time of \method?
	\item[Q3.] \textbf{Forgetting factor sensitivity (Section~\ref{subsec:forgetting_sensitivity}).} How does a forgetting factor affect the fitting to newly arrived data?
	\item[Q4.] \textbf{Anomaly detection (Section~\ref{subsec:experim_anomaly}).} What types of anomalies are there in a dual-way streaming setting?
\end{itemize*}


\begin{figure*}
	 \subfloat{\includegraphics[width=0.65\textwidth]{FIG/PERF/LEGEND.pdf}} \\
	\vspace{-3mm}	
	 \setcounter{subfigure}{0}
	\centering	
	 \subfloat[Running time on JPN Stock]{\includegraphics[width=0.18\textwidth]{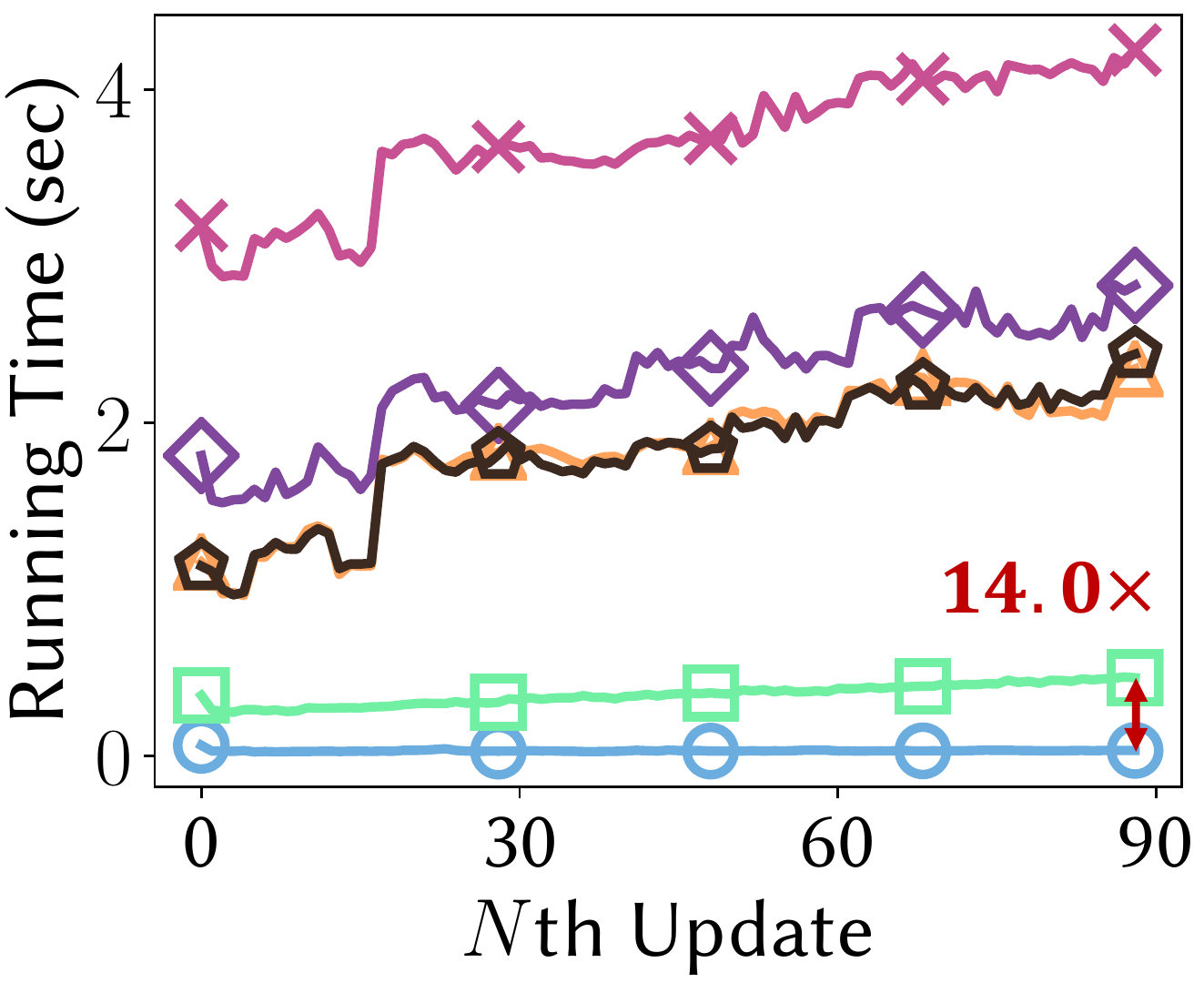}\label{fig:perf_jpn}} \quad\quad
	 \subfloat[Running time on CHN Stock]{\includegraphics[width=0.18\textwidth]{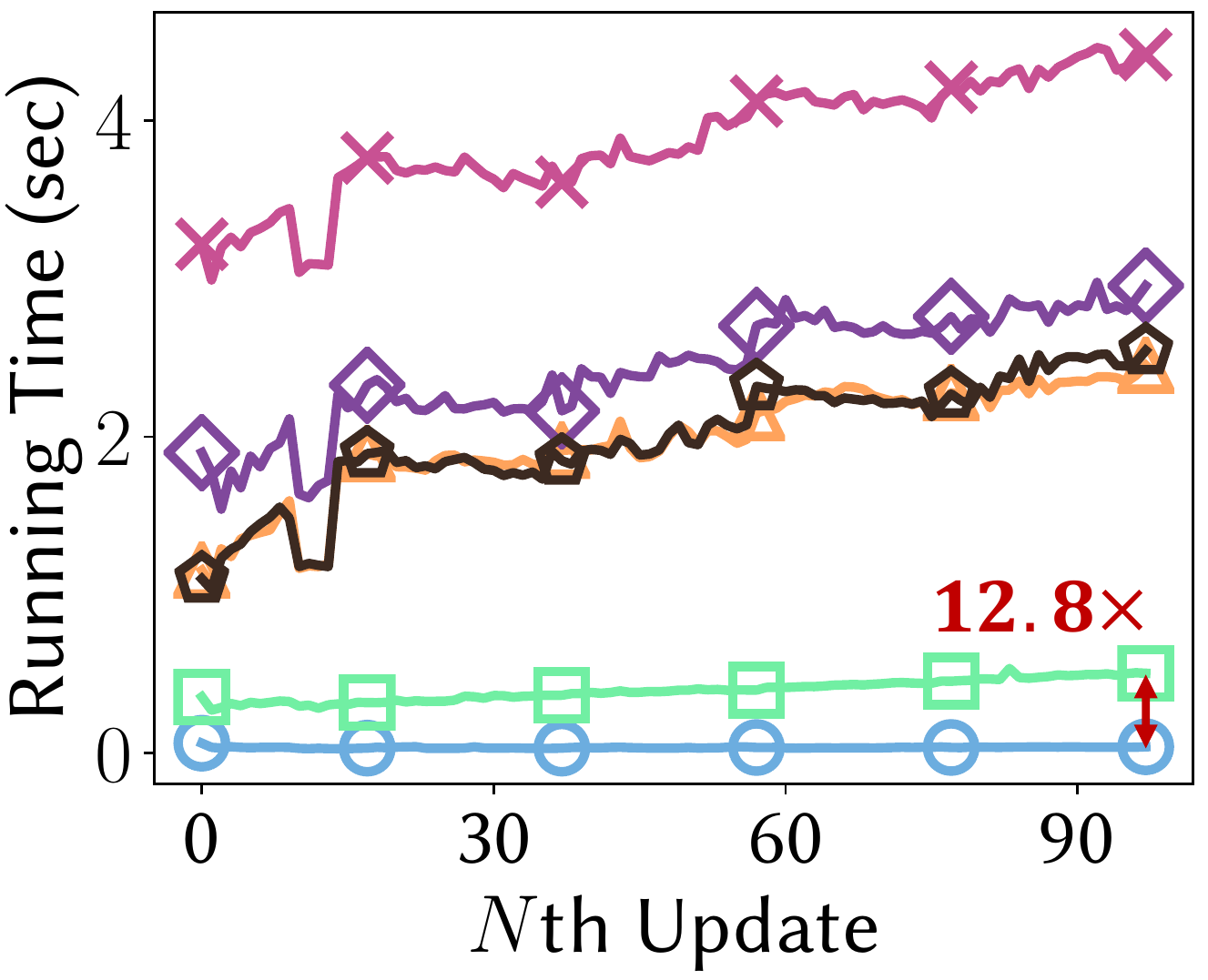}\label{fig:perf_chn}}	\quad \quad	 	
	 \subfloat[Running time on VicRoads]{\includegraphics[width=0.18\textwidth]{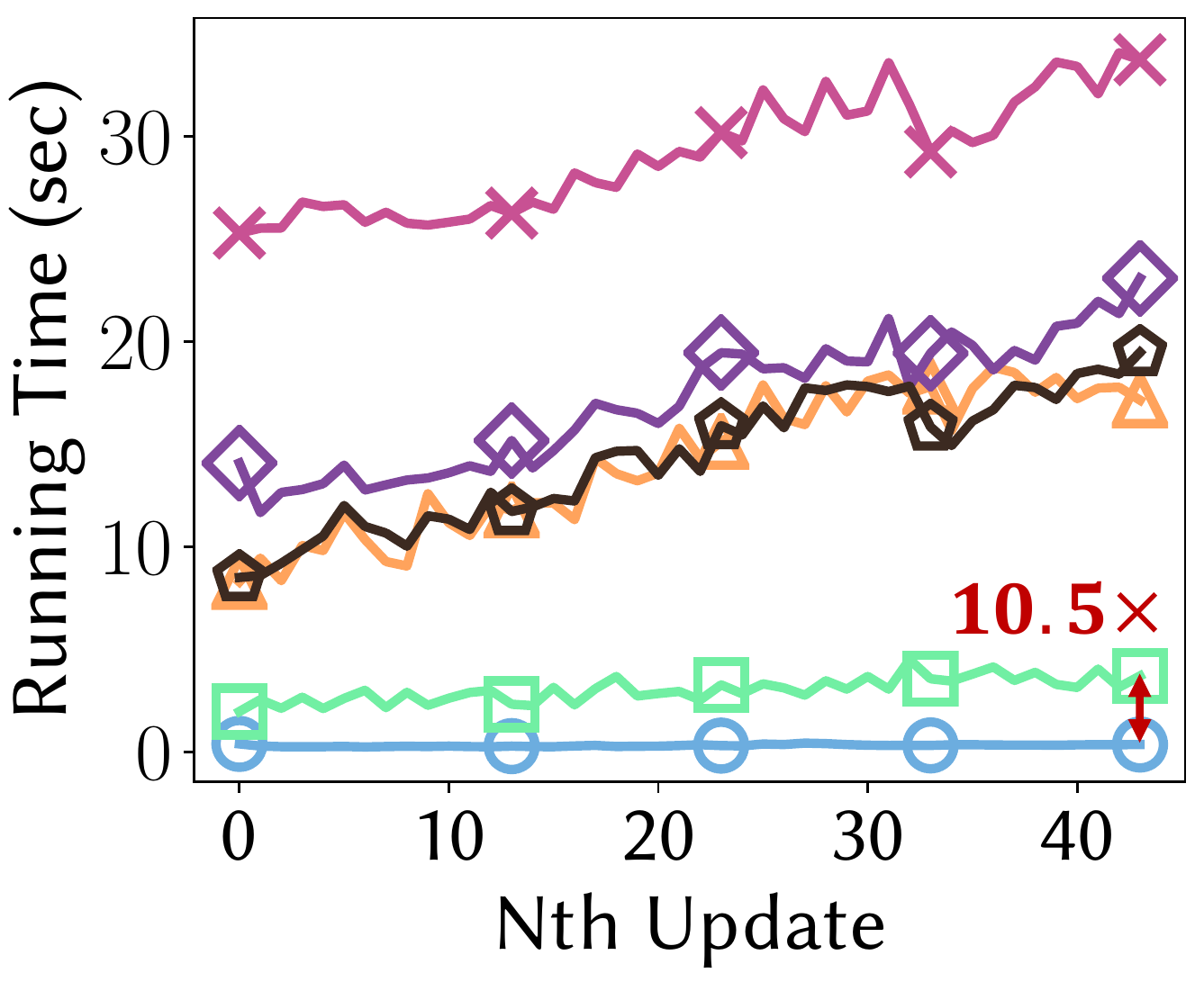}\label{fig:perf_traffic}}\quad\quad
	 \subfloat[Running time on PEMS]{\includegraphics[width=0.18\textwidth]{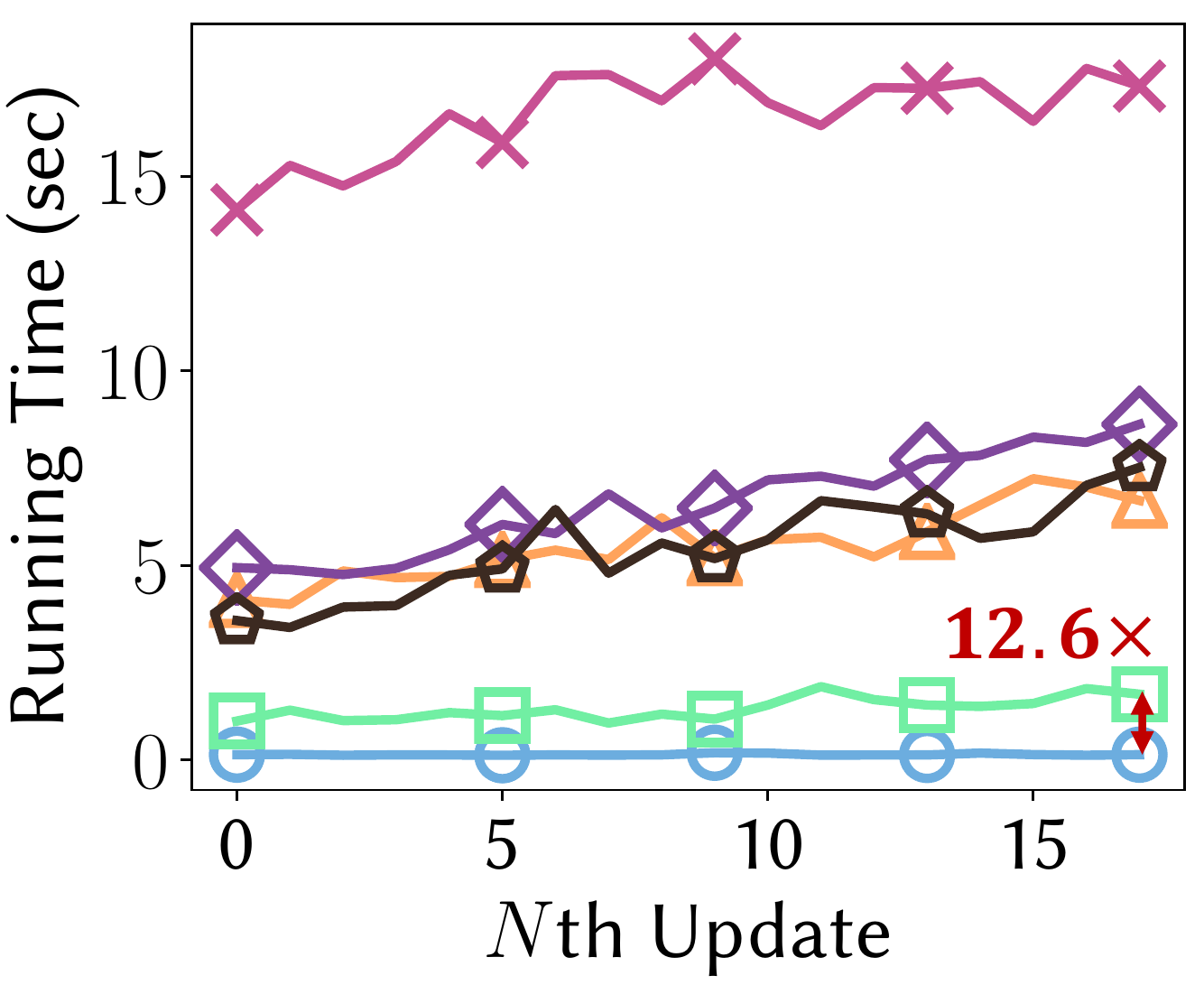}\label{fig:perf_pems}}	
	 	 \\
	\caption{Running time of \method and competitors for a new tensor on
real-world datasets.
	For each method, we draw a line connecting points of all updates, and five or six representative points with large markers.
	\method outperforms competitors with up to $14.0\times$ faster speed in the setting where only new rows of exiting slice matrices arrive over time.
	}
	\label{fig:perf}
\end{figure*}

\subsection{Experiment Setting}
\label{subsec:exp_setup}

\textbf{Machine.}
We conduct experiments on a workstation with 2 CPUs (Intel Xeon E5-2630 v4 @ 2.2GHz) and 512GB memory.

\textbf{Data.}
We use real-world datasets described in Table~\ref{tab:Description}.
The first four datasets are stock datasets.
Each stock dataset is a collection of stocks in United States\footnote{\url{https://datalab.snu.ac.kr/dpar2}\label{footnote:us}}, South Korea\footnote{\url{https://github.com/jungijang/KoreaStockData}\label{footnote:kr}}, Japan\footnote{\url{https://datalab.snu.ac.kr/atom/}\label{footnote:jpn}}, and China$^{\ref{footnote:jpn}}$ stock markets, respectively.
For stock datasets, each stock is represented as a slice matrix of an irregular tensor where rows represent the dates and columns represent the features.
There are 85 features: five basic features (opening price, closing price, highest price, lowest price, and trading volume) collected daily, and 80 technical indicators calculated from the basic features.
The next two datasets VicRoads\footnote{\url{https://github.com/florinsch/BigTrafficData}\label{footnote:vicroads}} and PEMS\footnote{\url{http://www.timeseriesclassification.com/}\label{footnote:pems}} contain traffic information.
For traffic datasets, each slice matrix corresponds to a location while its rows and columns represent the dates and timeframes, respectively.
Timeframes are fixed time periods throughout the day, which are 15 minutes long each for VicRoads and 10 minutes for PEMS.
As described in Table~\ref{tab:Description}, each dataset is included in one of the following two settings: 1) on US and KR Stock datasets, both new rows of existing slice matrices and new slice matrices arrive, and 2) on JPN Stock, CHN Stock, VicRoads, and PEMS datasets, only new rows of existing slice matrices arrive.
Appendix~\ref{appendix:tensor_size} describes the size of newly arrived data per update on US and KR Stock datasets.
For the other datasets, the size of newly arrived data is almost the same at each update since the number of slice matrices is constant.

\textbf{Initialization and update cycle.}
Before starting an update cycle, we initialize factor matrices of \method.
We use 20\% of the entire period as an initialization step, and the size of an initial tensor is described in Table~\ref{tab:Description}.
In addition, we choose update cycles depending on the total size of data as described in Table~\ref{tab:Description}.

\textbf{Competitors.}
We compare \method with existing PARAFAC2 decomposition methods which improve the efficiency of PARAFAC2 decomposition in other settings (e.g., static setting or a limited streaming setting) as there is no PARAFAC2 decomposition method designed for working in a dual-way streaming setting.
\begin{itemize*}
	\item \textbf{\als}: a base PARAFAC2 decomposition method for an irregular tensor.
	\item \textbf{\spartan}~\cite{PerrosPWVSTS17}: an efficient PARAFAC2 decomposition for irregular sparse tensors.
	\item \textbf{\rd}~\cite{cheng2019efficient}: a PARAFAC2 decomposition method with a direct fitting-based scheme.
	\item \textbf{\dpar}~\cite{JangK22}: an efficient PARAFAC2 decomposition method for irregular dense tensors.
	\item \textbf{\spade}~\cite{gujral2020spade}: an efficient PARAFAC2 decomposition method for newly arrived slice matrices. The implementation detail is described in Appendix~\ref{appendix:implement_spade}.
\end{itemize*}
Note that all the methods are implemented in MATLAB, and we use Tensor Toolbox~\cite{bader2019matlab}.

\begin{table*}[!t]
\centering
        \caption{
        Comparison for local errors.
        We measure mean and standard deviation for all updates.
        Bold and underlined values denote the lowest and the second-lowest local errors, respectively.
        \method achieves the lowest local errors than competitors. 
        }
        \label{tab:local_error}
                \resizebox{0.9\textwidth}{!}{
\begin{tabular}{lrrrrrrrr}
\toprule
\multicolumn{7}{c}{Local errors for newly arrived data}  \\
\midrule
\multicolumn{1}{l|}{Method}          & \multicolumn{1}{c}{US Stock} & \multicolumn{1}{c}{KR Stock} & \multicolumn{1}{c}{JPN Stock} & \multicolumn{1}{c}{CHN Stock} & \multicolumn{1}{c}{VicRoads} & \multicolumn{1}{c}{PEMS} \\
\midrule
\multicolumn{1}{l|}{\als}        & {$\underline{0.1130 \pm 0.0074}$}        & ${0.0809 \pm 0.0124}$            & ${0.1312 \pm 0.0080}$            & ${0.1333 \pm 0.0098}$ & $ {0.0519 \pm 0.0128}$ & $ {0.1137 \pm 0.0033}$  \\
\multicolumn{1}{l|}{\spartan}        & {${0.1131 \pm 0.0074}$}        & $\underline{0.0808 \pm 0.0122}$            & ${0.1312 \pm 0.0078}$            & ${0.1334 \pm 0.0101}$ & $ {0.0520 \pm 0.0128}$ & $ {0.1135 \pm 0.0033}$  \\
\multicolumn{1}{l|}{\rd}        & {${0.1141 \pm 0.0079}$}        & ${0.0822 \pm 0.0123}$            & ${0.1328 \pm 0.0087}$            & ${0.1352 \pm 0.0088}$ & $ \underline{0.0482 \pm 0.0126}$ & $ {0.1129 \pm 0.0035}$  \\
\multicolumn{1}{l|}{\dpar}        & {${0.1133 \pm 0.0072}$}        & ${0.0823 \pm 0.0122}$            & ${0.1331 \pm 0.0091}$            & ${0.1361 \pm 0.0093}$ & ${0.0491 \pm 0.0001}$ & $ \underline{0.1103 \pm 0.0030}$  \\
\multicolumn{1}{l|}{\spade}        & {${0.2225 \pm 0.0085}$}        & ${0.1554 \pm 0.0187}$            & $\underline{0.1311 \pm 0.0081}$            & $\underline{0.1331 \pm 0.0095}$ & $ {0.0520 \pm 0.0128}$ & $ {0.1138 \pm 0.0033}$  \\
 \midrule
\multicolumn{1}{l|}{\method (proposed)}        & {$\mathbf{0.1010 \pm 0.0047}$}        & $\mathbf{0.0724 \pm 0.0096}$            & $\mathbf{0.1185 \pm 0.0080}$            & $\mathbf{0.1196 \pm 0.0080}$ & $ \mathbf{0.0435 \pm 0.0121}$ & $ \mathbf{0.0912 \pm 0.0030}$  \\
\bottomrule
\end{tabular}}
\end{table*}

\begin{figure*}
	\centering	
	 \subfloat[Scalability on JPN Stock]{\includegraphics[width=0.2\textwidth]{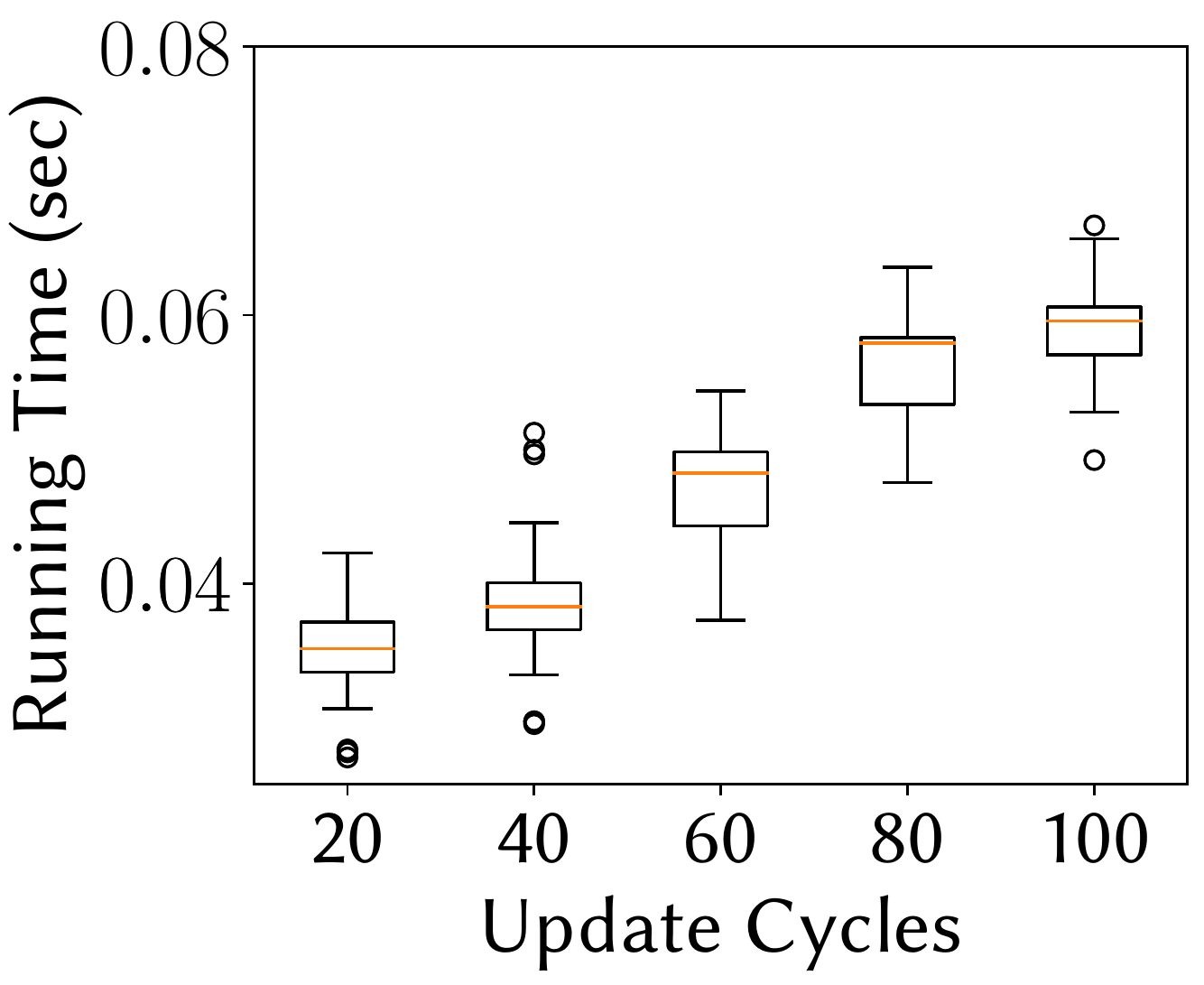}\label{fig:scala_jpn}} \quad\quad
	 \subfloat[Scalability on CHN Stock]{\includegraphics[width=0.2\textwidth]{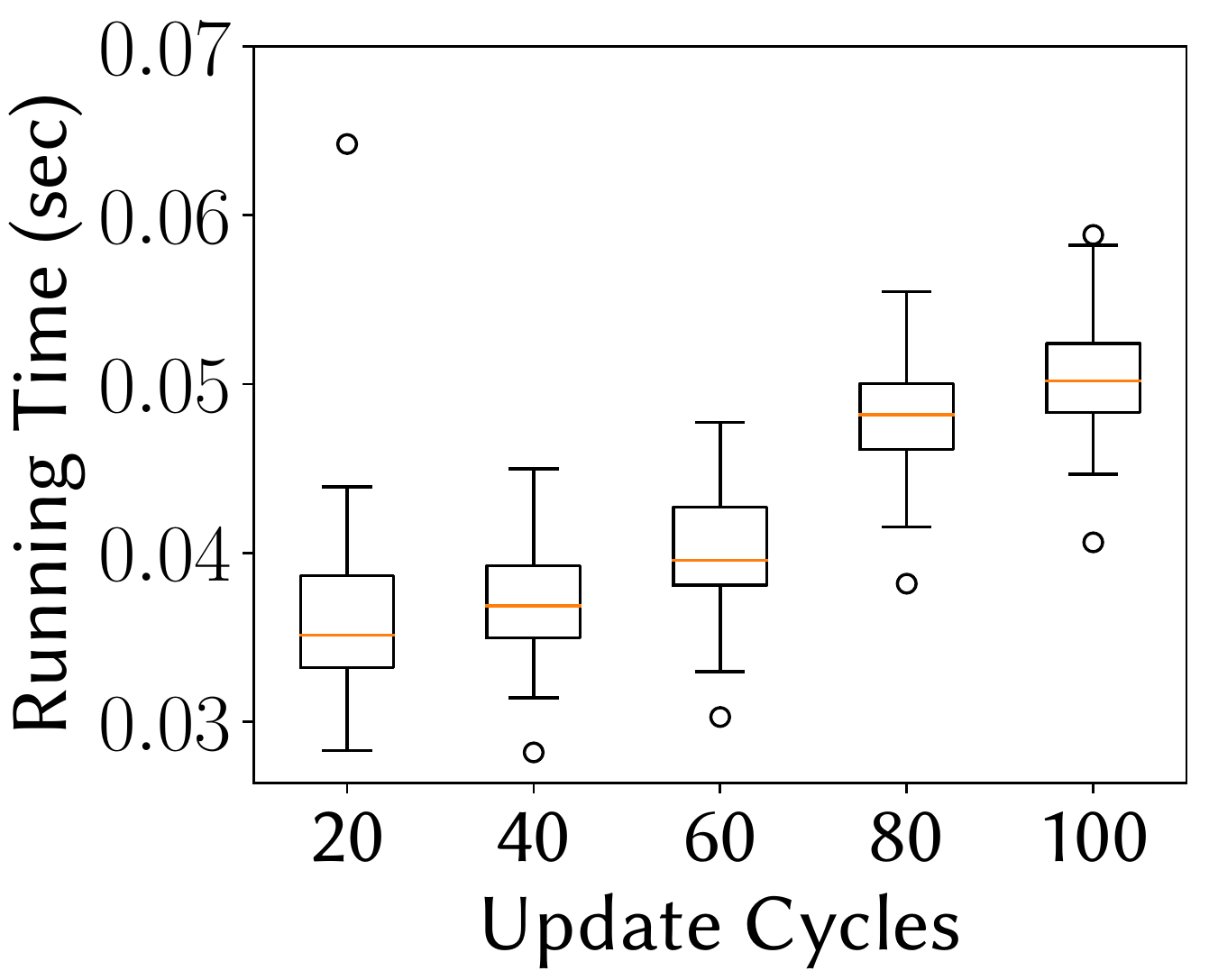}\label{fig:scala_chn}}	\quad\quad	
	 \subfloat[Scalability on VicRoads]{\includegraphics[width=0.2\textwidth]{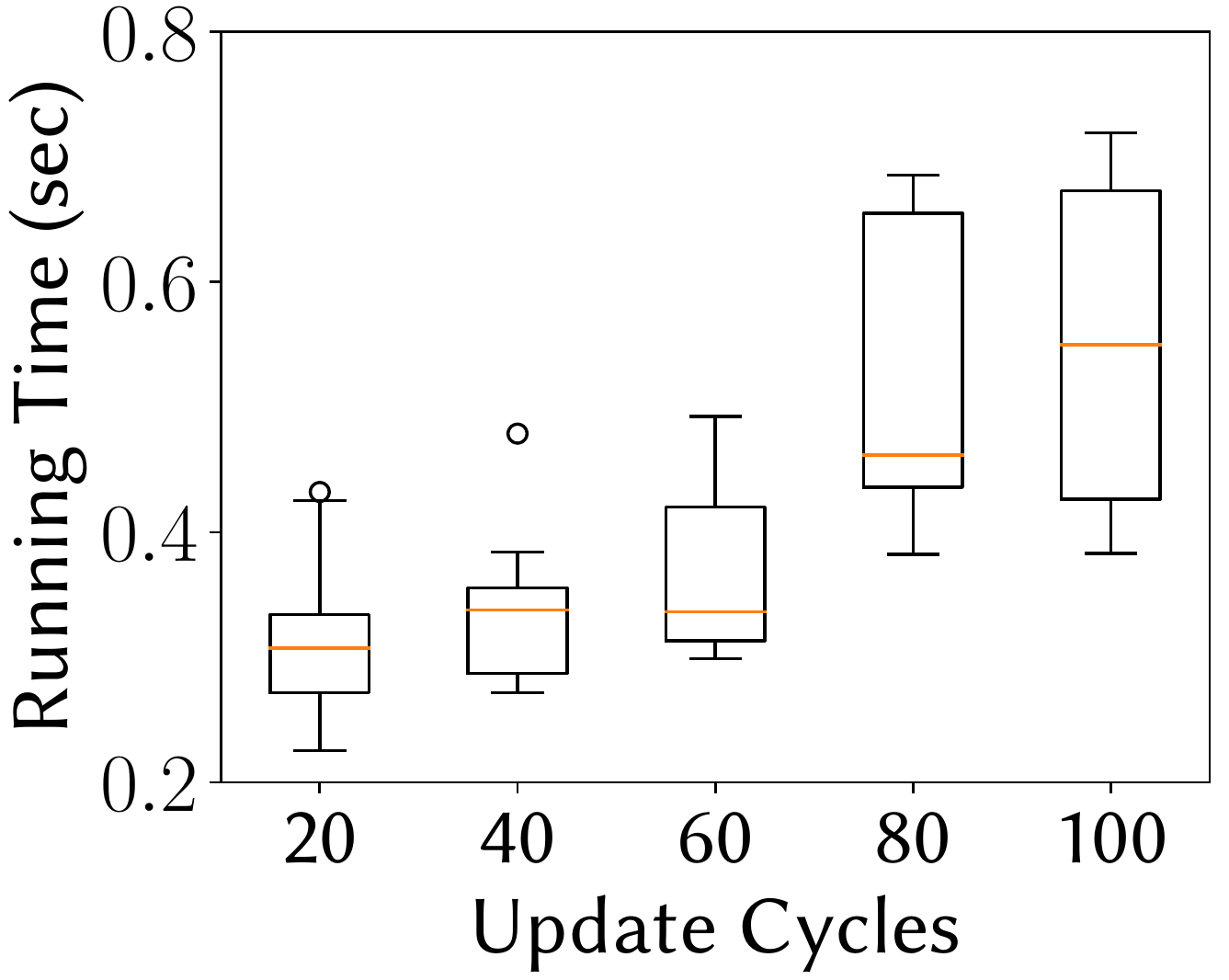}\label{fig:scala_traffic}} \quad\quad
	 \subfloat[Scalability on PEMS]{\includegraphics[width=0.195\textwidth]{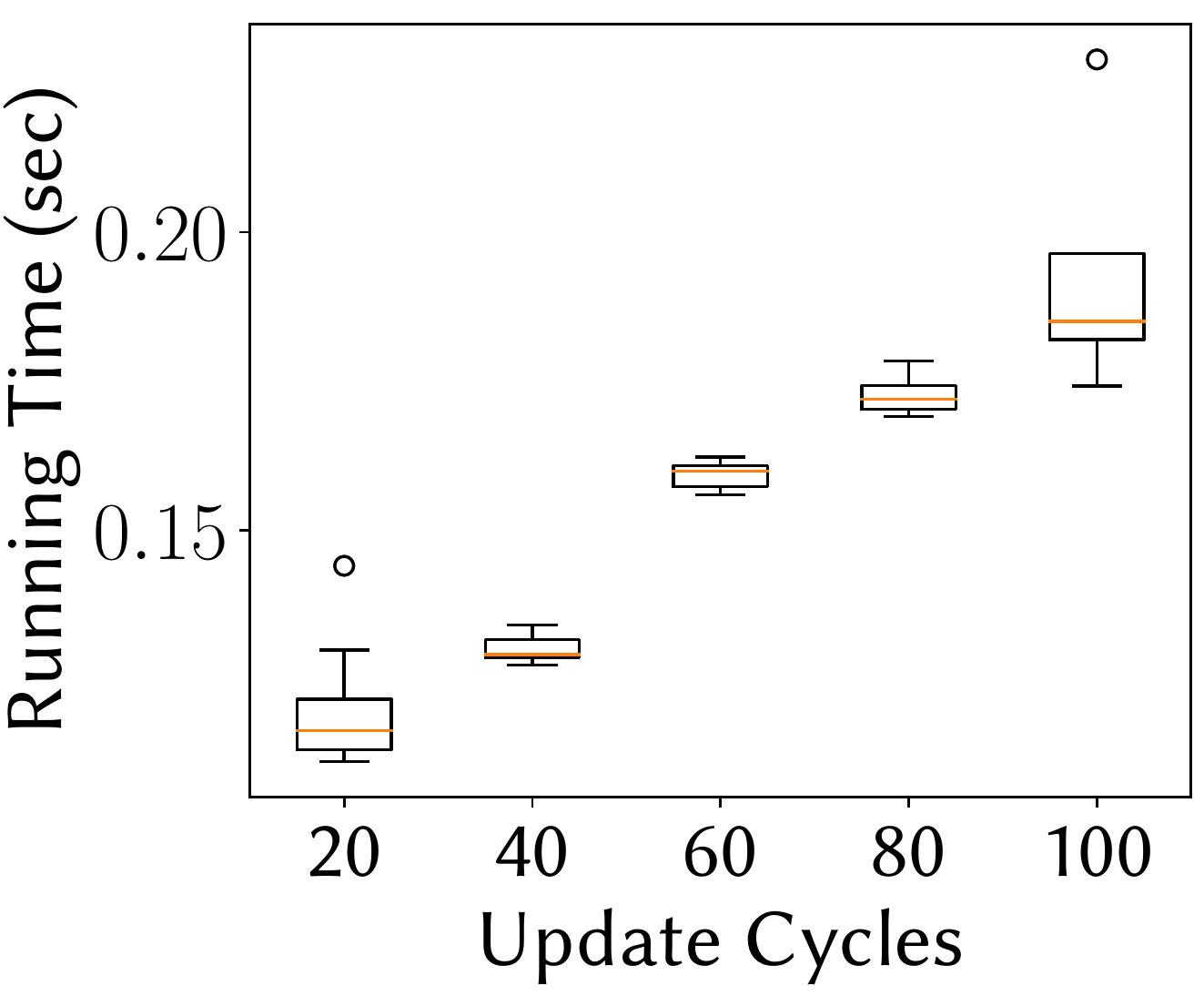}\label{fig:scala_pems}}	
	 	 \\
	\caption{Scalability with respect to five update cycles: $[20, 40, 60, 80, 100]$.
	The size of new rows of existing slice matrices is linearly proportional to an update cycle.
	Note that the running time of \method increases linearly with the update cycles.
	}
	\label{fig:scalability}
\end{figure*}

\textbf{Hyperparamter setting.}
There are several hyperparameters used in experiments: target rank $R$, maximum iteration $N$, and a forgetting factor $\lambda$.
We set the target rank $R$ and maximum iteration to $10$.
Except in Section~\ref{subsec:forgetting_sensitivity}, we set the forgetting factor $\lambda$ to $0.7$.

\textbf{Normalization.}
We normalize an initial tensor and newly arrived data using a min-max scaler for each column of slice matrices.
As an example of a new slice matrix $\mat{X}_{k,new}$, we normalize $\mat{X}_{k,new}(:,j)$ to $\frac{\mat{X}_{k,new}(:,j) - minimum_{k,j}}{maximum_{k,j} - minimum_{k,j}}$ where $maximum_{k,j}$ and  $minimum_{k,j}$ are maximum and minimum values of the $j$th column of the $k$th slice matrix $\mat{X}_{k,new}$, respectively.

\textbf{Local and global errors.}
We use two reconstruction errors for newly arrived data and an accumulated tensor, respectively.
A local error is computed by Definition~\ref{def:tensor_level} while a global error is computed as follows:
%
\vspace{-1mm}
\begin{align*}
\small
\begin{split}
	&\text{Global Error} = \frac{1}{K}\sum_{k=1}^{K}{\frac{1}{I_{k,old}\times J}\sum_{i,j}{|\mat{X}_{k,old}(i,j) - \hat{\mat{X}}_{k,old}(i,j)|}}  \\
	&+ \frac{1}{K+L}\sum_{k=1}^{K+L} {\frac{1}{I_{k,new}\times J} \sum_{i,j}{| \mat{X}_{k,new}(i,j) - \hat{\mat{X}}_{k,new}(i,j) |}}
	\end{split}
\end{align*}
where $I_{k,old}$ is the row length of the $k$th slice matrix of a tensor accumulated until the previous update.
$\hat{\mat{X}}_{k,old}$ is reconstructed by $\mat{U}_{k,old}\mat{S}_{k}$ $\mat{V}^T$.

\subsection{Performance for Newly Arrived Data (Q1)}
\label{subsec:performance}

We compare the performance of \method with that of competitors in terms of efficiency and local errors.

\textbf{Efficiency.}
We compare the performance of \method with that of competitors in a dual-way streaming setting.
In Figures, the running time for the $N$th update indicates an update time for the $N$th newly arrived data, not accumulated time.
For US and KR Stock datasets, Figures~\ref{fig:perf_us} and~\ref{fig:perf_kr} show the results in the first setting where both new rows of existing slice matrices and new slice matrices arrive.
\method outperforms the existing static PARAFAC2 decomposition methods and the streaming PARAFAC2 decomposition method, by up to $12.9\times$ faster speed than competitors.
For JPN Stock, CHN Stock, VicRoads, and PEMS datasets, Figure~\ref{fig:perf} shows that \method achieves up to $14.0\times$ faster speed than competitors in the second setting where new rows of slice matrices arrive.

\textbf{Local Error.}
We evaluate the performance of \method in terms of local errors,
by measuring their mean and standard deviation for all updates.
Table~\ref{tab:local_error} shows that \method has much lower local errors than competitors for all the datasets.
This is because \method fits factor matrices to newly arrived data using a forgetting factor $\lambda$.
In a dual-way streaming setting, reducing local errors is more crucial than reducing global errors since new data arrive infinitely.
%
%
The comparison for global errors is described in Appendix~\ref{appendix:errors}.

\subsection{Scalability for Newly Arrived Data Size (Q2)}
\label{subsec:scalability}
To evaluate the scalability with respect to the size of a new incoming tensor, we measure the running time for five update cycles: $[20, 40, 60,$ $ 80, 100]$.
The length of new rows of existing slice matrices is linearly proportional to an update cycle, and the number of new slice matrices at each update is different depending on an update cycle.
We report results using box plots where an orange line denotes the median, a box is constructed by two quartiles $Q3$ and $Q1$, two horizontal lines indicate $2.5*Q3 - 1.5*Q1$ and $2.5*Q1 - 1.5*Q3$, and points indicate outliers.
Figures~\ref{fig:scalability_us},~\ref{fig:scalability_kr}, and~\ref{fig:scalability} show that the running time of \method is linearly proportional to update cycles.
Note that Figures~\ref{fig:scalability_us},~\ref{fig:scalability_kr} show the experimental results for the setting where new matrices and new rows arrive, while Figure~\ref{fig:scalability} shows the ones for the one where only new rows arrive.

\subsection{Forgetting Factor Sensitivity (Q3)}
\label{subsec:forgetting_sensitivity}
We evaluate the sensitivity of a forgetting factor by varying its hyperparameter $\lambda$ in $[0.1, 0.3, 0.5, 0.7, 0.9]$.
For each forgetting factor, we measure global and local errors for an accumulated tensor and newly incoming data, respectively, and then average each of them.
Figure~\ref{fig:sensitivity} shows that the forgetting factor provides trade-offs between local and global errors.
As the value of the forgetting factor increases, the local errors increase while the global errors decrease.
The lowest and the highest forgetting factors (i.e., $0.1$ and $0.9$) provoke high global errors and low local errors, respectively, so that choosing them should be avoided.
If we focus more on reducing local errors, we consider using a forgetting factor as $0.7$.

\begin{figure*}
	 \subfloat{\includegraphics[width=0.4\textwidth]{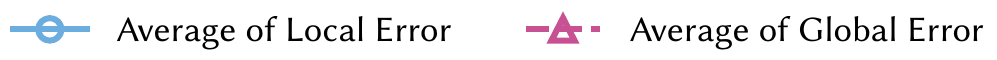}} \\
	\vspace{-3mm}	
	 \setcounter{subfigure}{0}
	\centering	
	 \subfloat[Sensitivity on US Stock]{\includegraphics[width=0.15\textwidth]{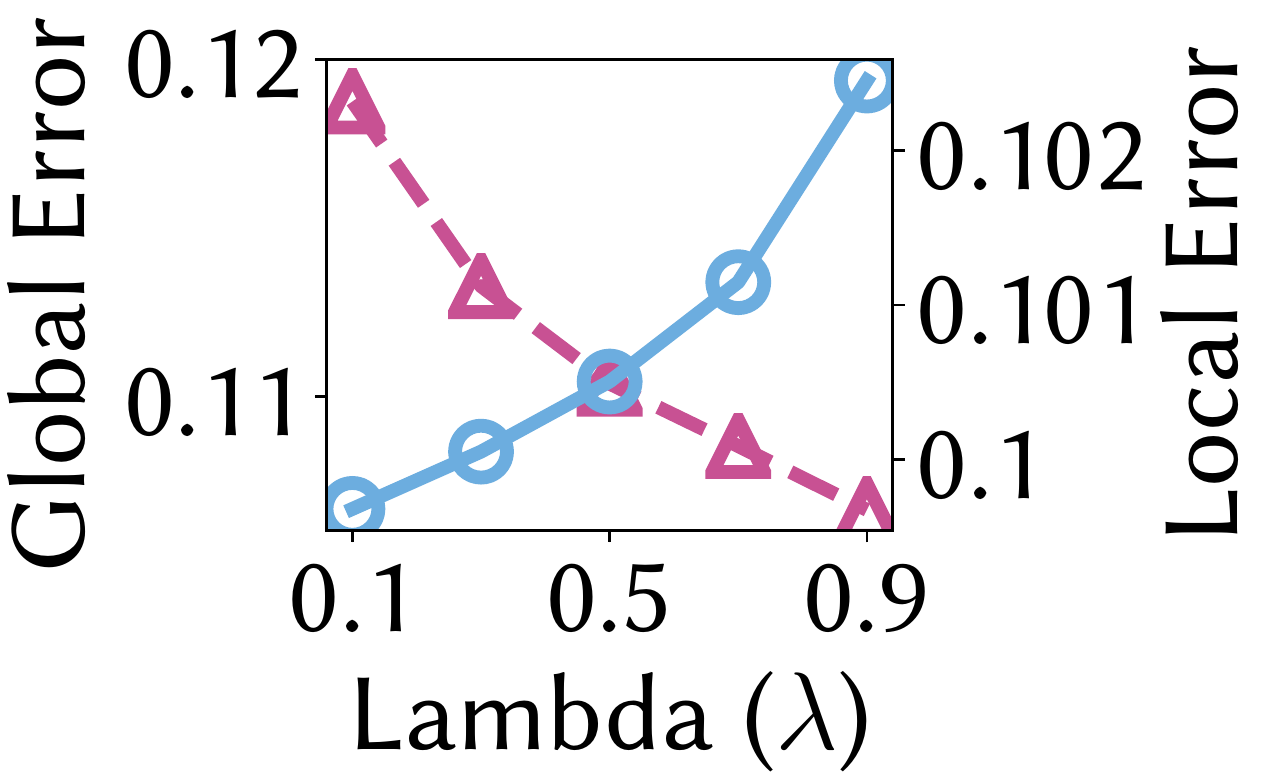}\label{fig:sensitivity_us}}
	 \subfloat[Sensitivity on KR Stock]{\includegraphics[width=0.155\textwidth]{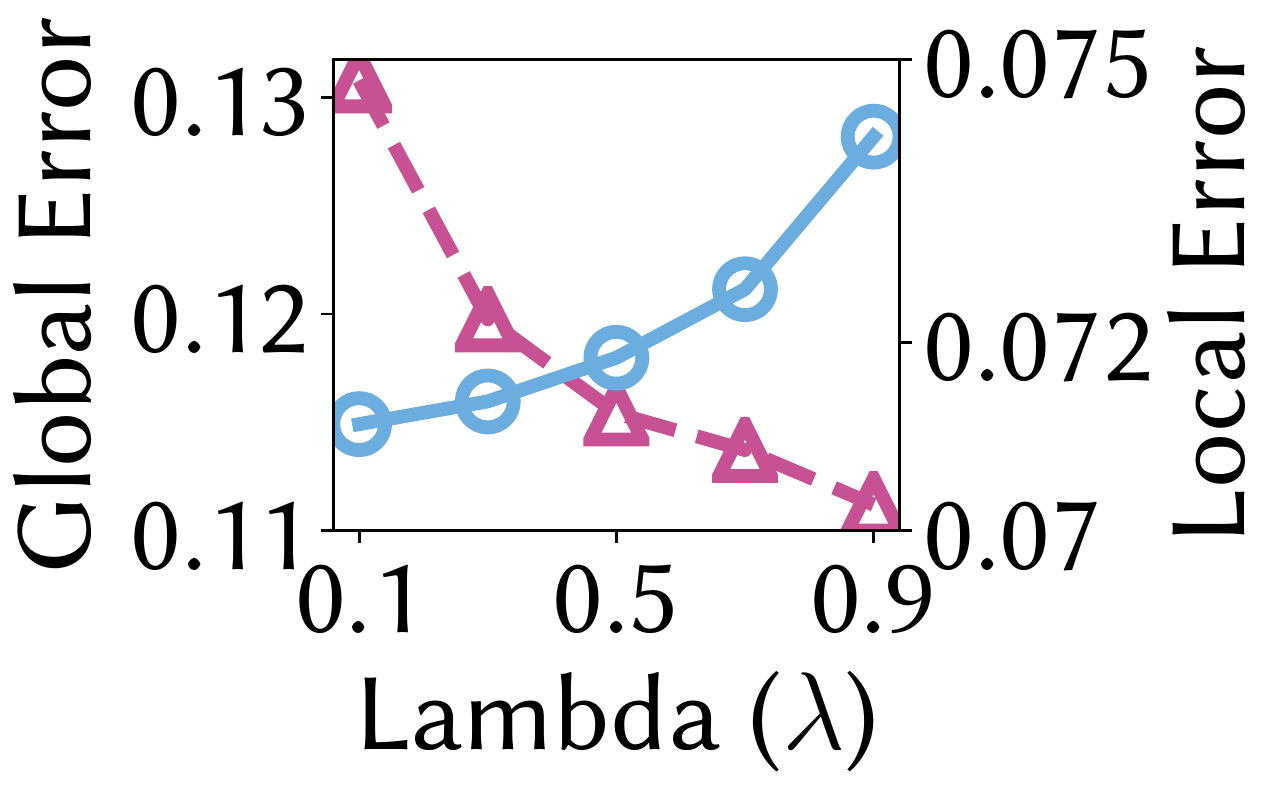}\label{fig:sensitivity_kr}}
	 \subfloat[Sensitivity on JPN Stock]{\includegraphics[width=0.15\textwidth]{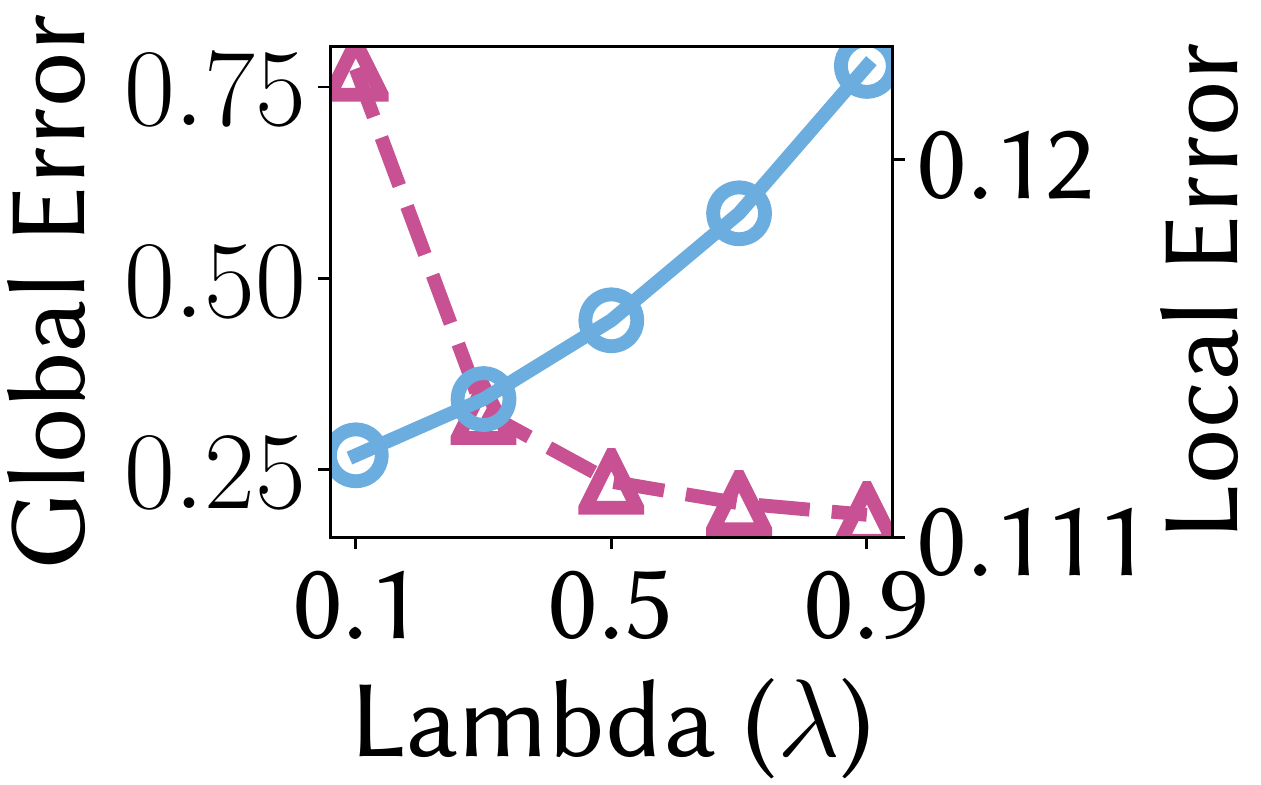}\label{fig:sensitivity_jpn}}
	 \subfloat[Sensitivity on CHN Stock]{\includegraphics[width=0.155\textwidth]{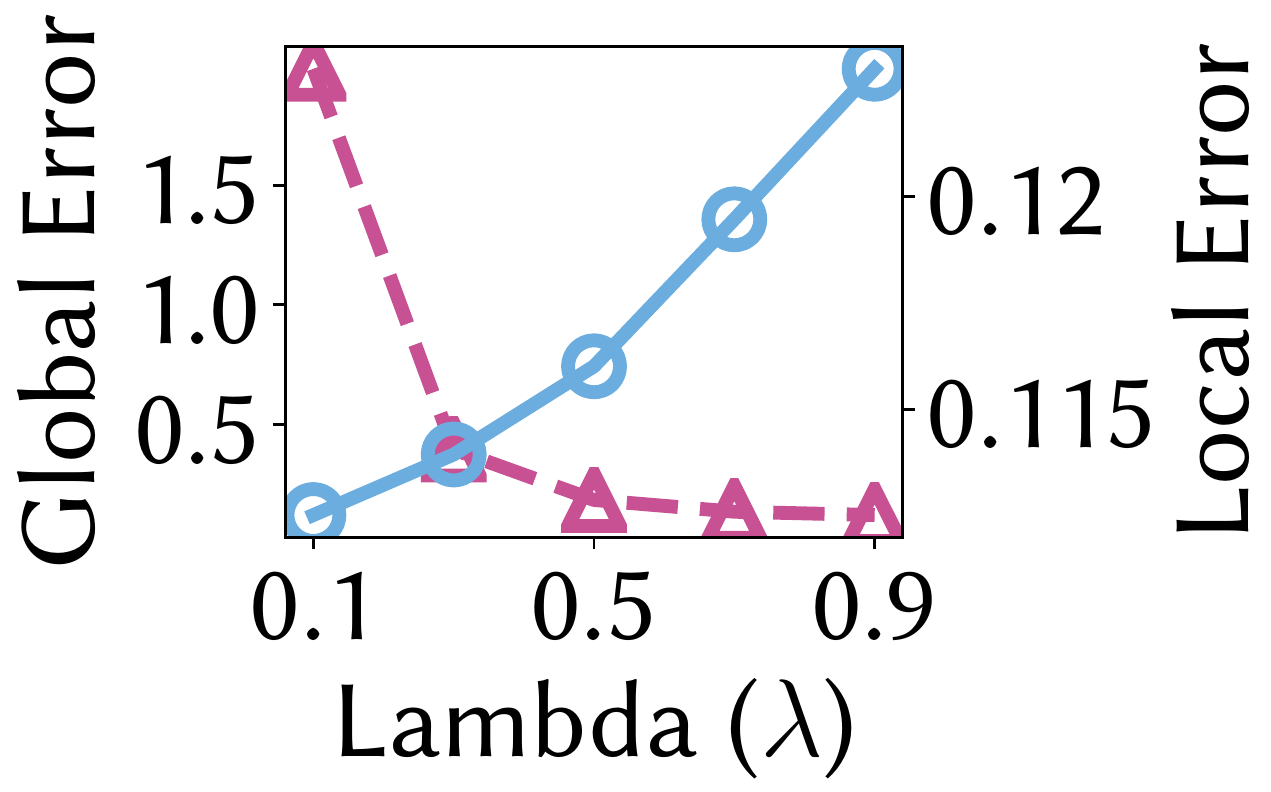}\label{fig:sensitivity_chn}}	
	 \subfloat[Sensitivity on VicRoads]{\includegraphics[width=0.165\textwidth]{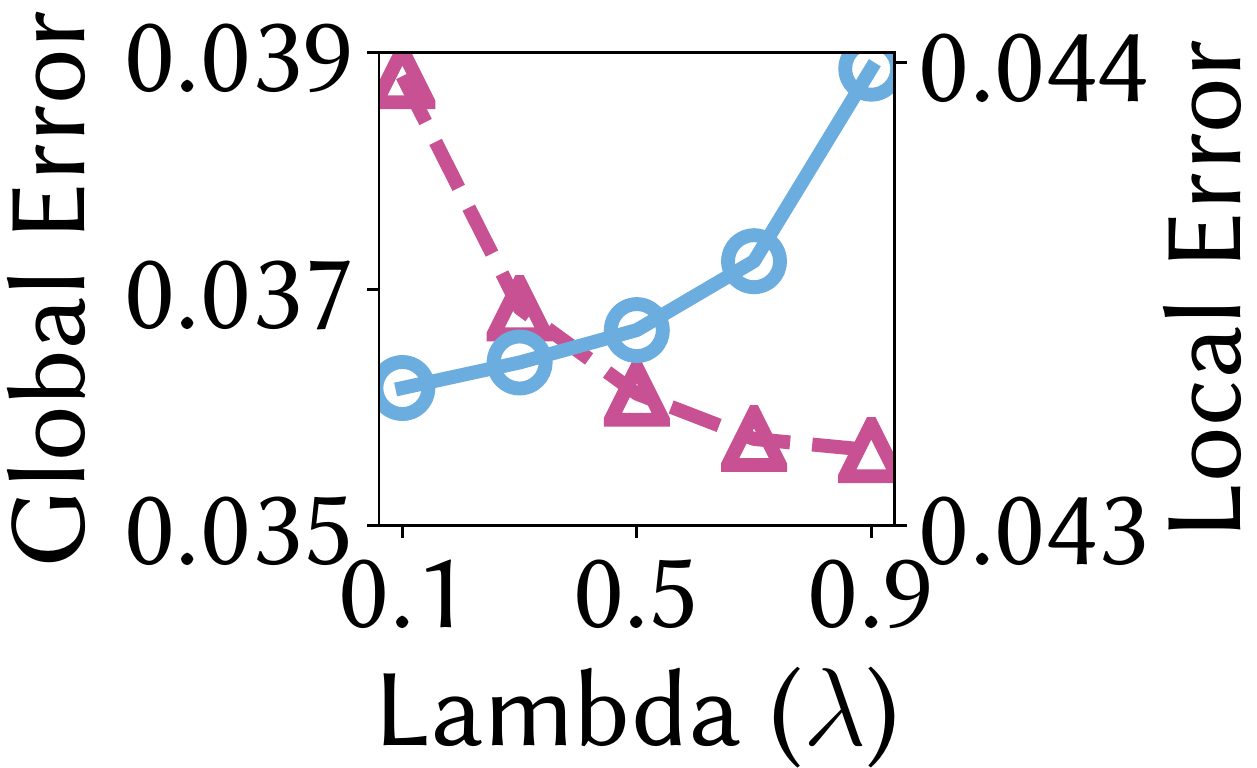}\label{fig:sensitivity_traffic}}
	 \subfloat[Sensitivity on PEMS]{\includegraphics[width=0.16\textwidth]{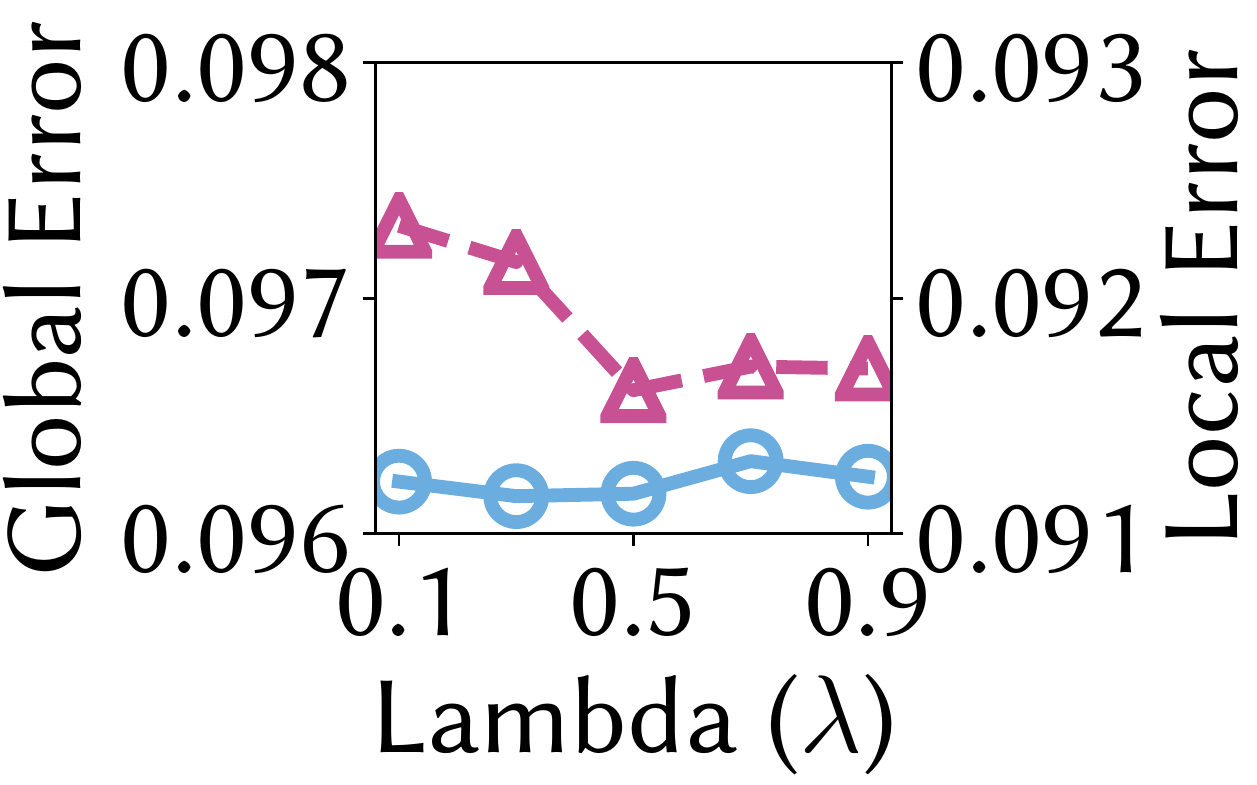}\label{fig:sensitivity_pems}}		  	 	
	 	 \\
	\caption{Error measurement with respect to forgetting factors $\lambda$: $[0.1, 0.3, 0.5, 0.7, 0.9]$.
	Global and local errors are computed for an accumulated tensor and a new incoming tensor, respectively.
	The forgetting factor provides trade-offs between global and local errors.
The values of the forgetting factor around $0.1$ and $0.9$ provoke high global error and high local error, respectively.
	We set it to $0.7$ since local errors highly increase between $0.7$ and $0.9$.
	}
	\label{fig:sensitivity}
\end{figure*}

\begin{figure}
	\centering	
	\includegraphics[width=0.24\textwidth]{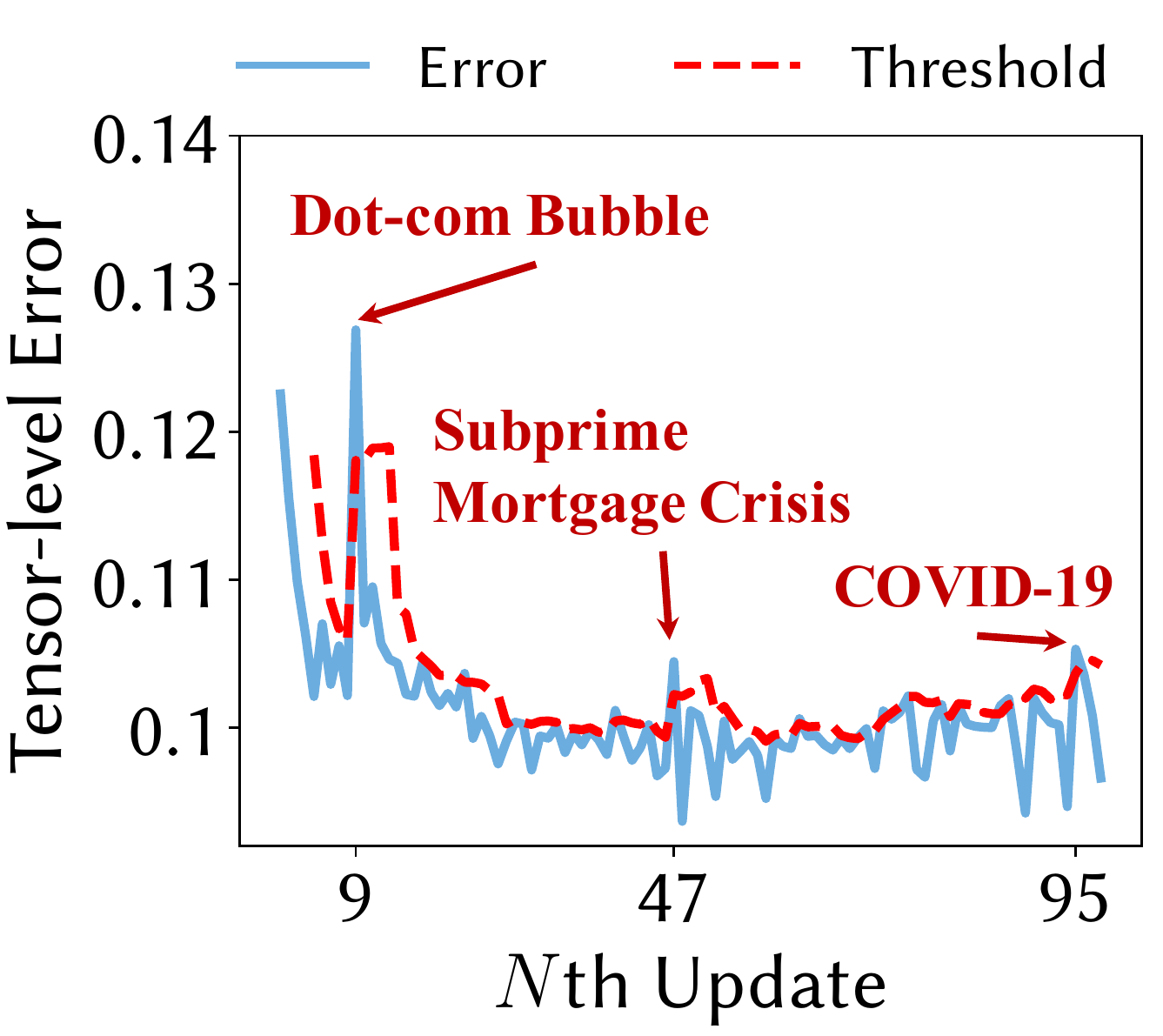}
	 	 \\
	\caption{Tensor-level anomaly detection.
	We find three tensor-level anomalies which are linked to big events in the stock market: the Dot-com Bubble, Subprime Mortgage Crisis, and COVID-19.
	}
	\label{fig:tensorlevel_anomaly}
\end{figure}

\subsection{Anomaly Detection (Q4)}
\label{subsec:experim_anomaly}

We present case studies of detecting anomalies using \method.
We analyze US Stock dataset, and detect two types of anomalies: tensor-level and slice-level anomalies.
At each update, we set the sum of the moving average and the moving standard deviation of the window size $5$ as a threshold.
Note that an update cycle is $60$ days.

\textbf{Tensor-level anomaly detection.}
We discover tensor-level anomalies by measuring local errors defined in Definition~\ref{def:tensor_level} for all updates.
Figure~\ref{fig:tensorlevel_anomaly} shows the top-$3$ anomalies out of $12$ tensor-level anomalies which have higher errors than the threshold.
The first anomaly is discovered when \method updates factor matrices for new data collected between Oct. 29, 1999 to Jan 25, 2000.
After this period associated with the Dot-com Bubble, the S\&P 500 index hit its peak in March 2000 and plummeted since then.
The second anomaly is discovered when \method updates factor matrices for new data collected between Nov. 21, 2008 to Feb. 19, 2009.
After this period related to the Subprime Mortgage Crisis, the S\&P 500 index hit its lowest level in March 2009 and soared since then.
The third anomaly is discovered 
between May 4, 2020 to July 28, 2020.
The S\&P 500 index recovered in this period after it plummeted until March 2020.
All three anomalies occur around a large turning point.

\textbf{Slice-level anomaly detection.}
We discover slice-level anomalies for Microsoft stock whose data correspond to a slice matrix.
At each update, we measure a slice-level error defined in Definition~\ref{def:slice_level}, and identify an anomaly based on a threshold.
Figure~\ref{fig:slicelevel_anomaly} shows that the slice-level error is much higher than the threshold when \method updates factor matrices for new data collected in Jan. 13, 2006 - April 10, 2006.
We further analyze the abnormal period.
As shown in Figure~\ref{fig:msft_price}, Microsoft stock has a smaller difference between the highest and the lowest prices in this period than in the previous period (July 25, 2005 to Jan. 12, 2006): the difference is $2.16\$$ in this period while the difference is $3.7\$$ in the previous period.
Interestingly, after the abnormal period, the price plunged more than 10\% due to the quarterly earnings report and then recovers.

%% file: 050related.tex

\textbf{Irregular tensor decomposition in a static setting.}
Many previous works have proposed effective PARAFAC2 decomposition methods working in a static setting.
SPARTan~\cite{PerrosPWVSTS17} is an efficient PARAFAC2 decomposition method for analyzing EHR (Electronic Health Record) data.
COPA~\cite{AfsharPPSHS18} is an effective method which utilizes useful constraints in the objective function of PARAFAC2 decomposition.
REPAIR~\cite{ren2020robust} is a robust method for irregular tensors having missing values and erroneous values.
LogPar~\cite{yin2020logpar} handles a binary irregular tensor using PARAFAC2 decomposition.
TedPar~\cite{yin2021tedpar} develops an effective PARAFAC2 decomposition method by exploiting temporal dependency inherent in irregular tensors.
RD~\cite{cheng2019efficient} and DPar2~\cite{JangK22} are efficient PARAFAC2 decomposition methods for irregular tensors.
Jang et al.~\cite{Jang2022accurate} propose an accurate PARAFAC2 decomposition method for a temporal irregular tensor with missing values.
APTERA~\cite{gujralaptera} finds the target rank of PARAFAC2 decomposition automatically.
Although the above methods work effectively in a static setting,
they fail to deal with a dual-way streaming setting where existing slice matrices grow over time and new slice matrices continuously arrive.
The running time of the above methods explosively increases as tensor data are accumulated over time.
In contrast to the aforementioned methods, \method efficiently updates factor matrices for a newly arrived irregular tensor.

\textbf{Tensor decomposition in a streaming setting.}
Many previous tensor decomposition methods efficiently update factor matrices in a streaming setting where data are collected over time.
%
Online streaming tensor decomposition methods~\cite{NionS09,zhou2016accelerating,zhou2018online,jang2022static} have attracted much attention to efficiently analyze a tensor when the size of the time dimension of a tensor increases over time.
%
Ahn et al.~\cite{AhnKK21} and Lee et al.~\cite{LeeS21} utilize the temporal patterns of streaming tensors.
Pasricha et al.~\cite{PasrichaGP18} and Son et al.~\cite{son2022dao} incorporate the ideas of concept drift and change point detection in streaming tensor decomposition, respectively.
Soh et al.~\cite{SohFLSCPC21} propose an efficient CP decomposition method for a streaming tensor on parallel systems.
In addition to the above methods which are applicable when only one mode increases,
many recent works~\cite{song2017multi,najafi2019outlier,NimishakaviMGT18,XiaoWMG18,YangY20,YangGSZC21} extend the streaming setting even further to handle incoming data in more than one direction.
However, none of the methods mentioned above considers handling irregular tensors.
In contrast to the above methods, \method deals with irregular tensors in a streaming setting.
SPADE~\cite{gujral2020spade} efficiently updates factor matrices for new slice matrices of an irregular tensor, but fails to efficiently handle newly arrived data of existing slice matrices.
\method is the only tensor decomposition method to efficiently deal with both new rows of existing slice matrices and new slice matrices of an irregular tensor.

\begin{figure}
	\centering	
	 \subfloat[Slice-level anomaly detection]{\includegraphics[width=0.23\textwidth]{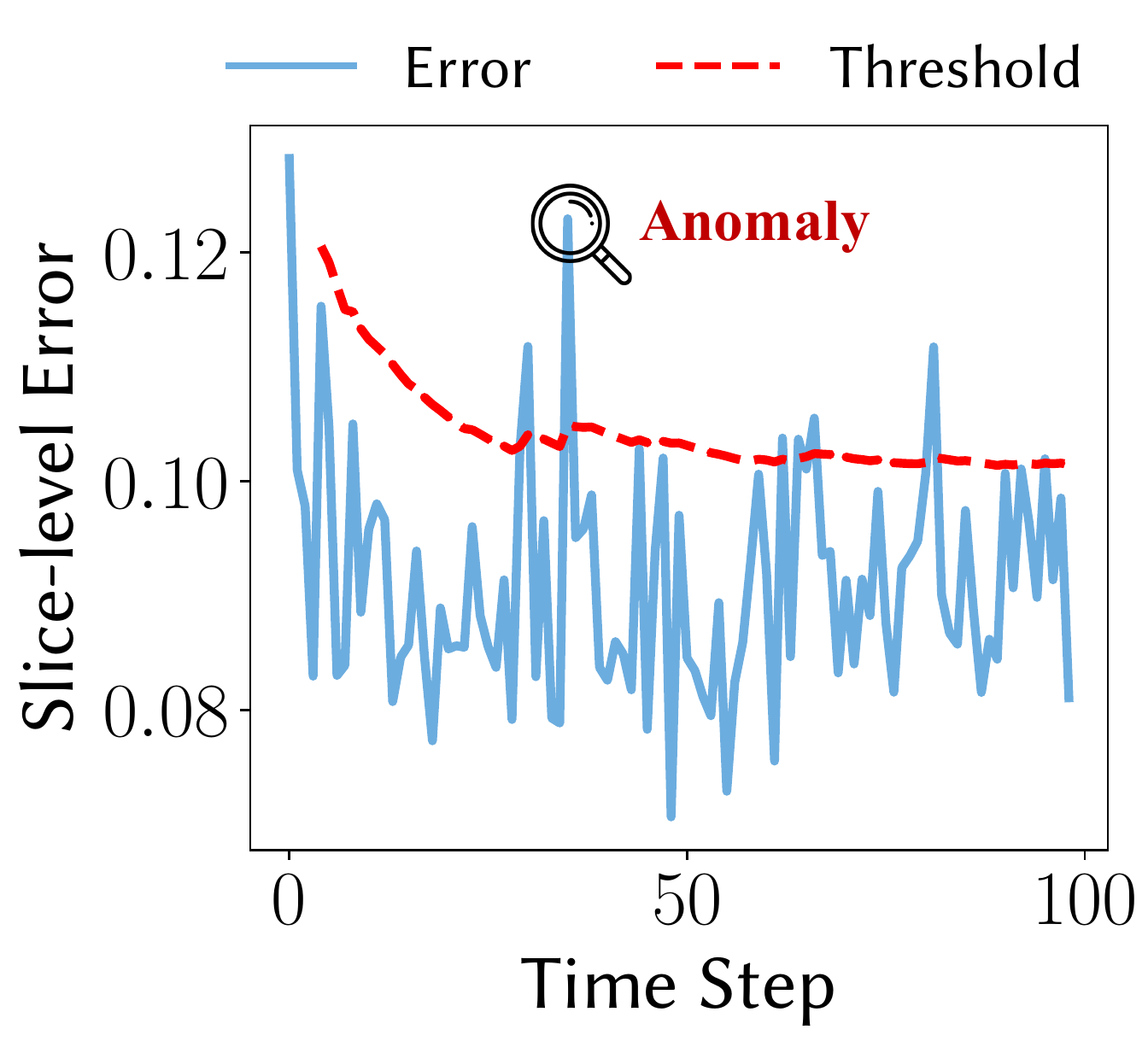}\label{fig:slicelevel_anomaly}}	
	 \subfloat[The price of Microsoft stock]{\includegraphics[width=0.23\textwidth]{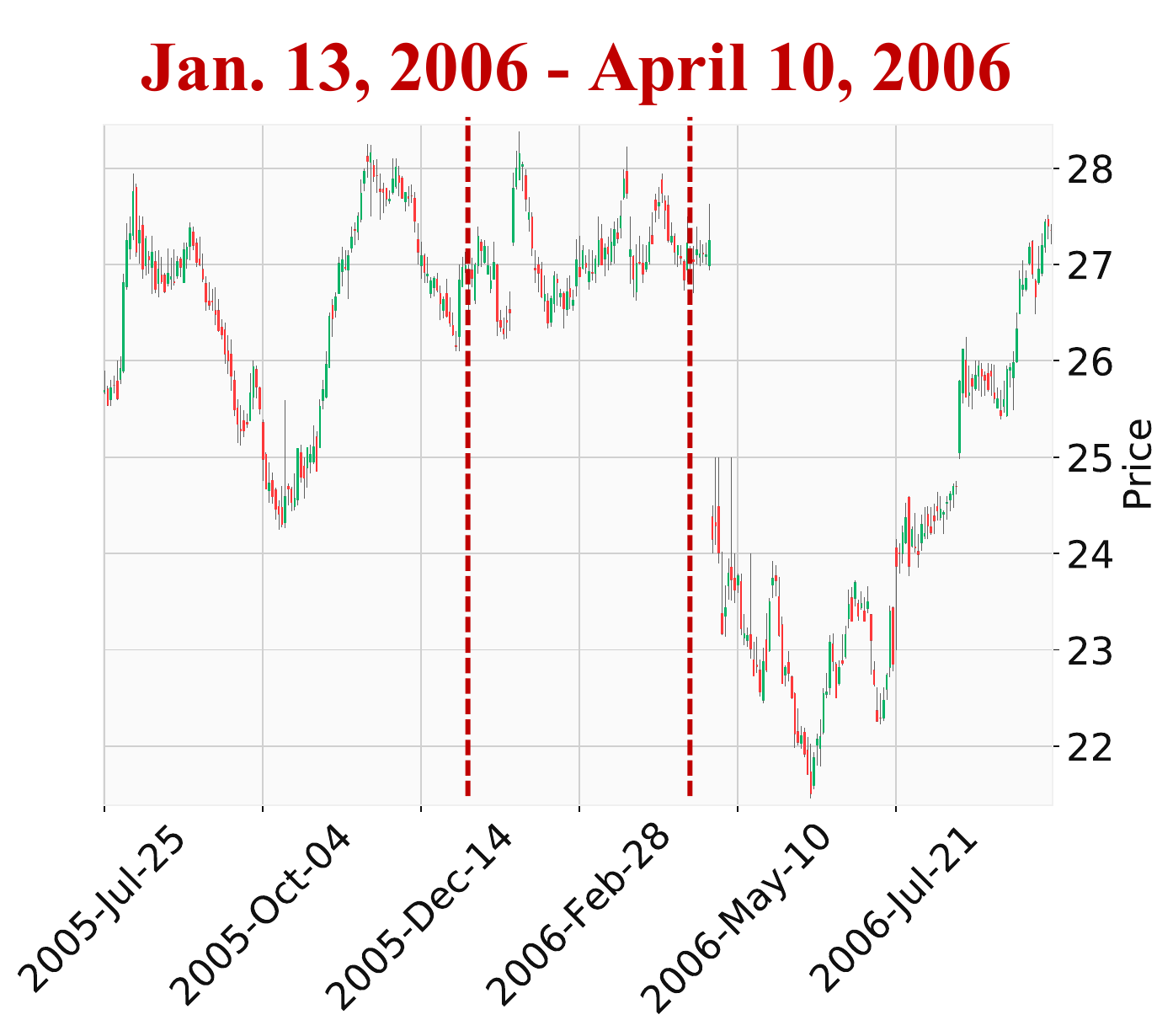}\label{fig:msft_price}}	
	\caption{Slice-level anomaly detection.
	(a) Top-$1$ anomaly of the Microsoft stock, detected by \method.	
	(b) The prices of Microsoft stock between July 25, 2005 and Sep. 29, 2006.
	}
	\label{fig:msft}
\end{figure}

%% file: 060conclusion.tex
We propose \method, a fast and accurate PARAFAC2 decomposition method in a dual-way streaming setting where new rows of matrices and new matrices simultaneously arrive over time.
We divide the terms related to old and new data and then propose an efficient update rule that allows us to avoid computing the terms involved with old data.
Furthermore, we add a forgetting factor so that \method fits factor matrices more to a newly arrived tensor while discarding old information gradually.
We experimentally show that \method is up to $14.0\times$ faster than existing PARAFAC2 decomposition methods in the dual-way streaming setting.
Thanks to \method, we discover various types of anomalies in a real-world dataset,
including Subprime Mortgage Crisis and COVID-19.
Our future direction is to adjust a forgetting factor adaptively considering the characteristics of newly arrived data.


%% file: 080appendix.tex
\begin{figure*}[t!]
	 \subfloat{\includegraphics[width=0.4\textwidth]{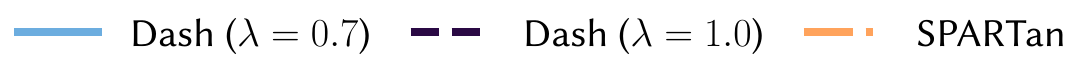}} \\
	\vspace{-3mm}	
	 \setcounter{subfigure}{0}
	\centering	
	 \subfloat[Global Error on US Stock]{\includegraphics[width=0.16\textwidth]{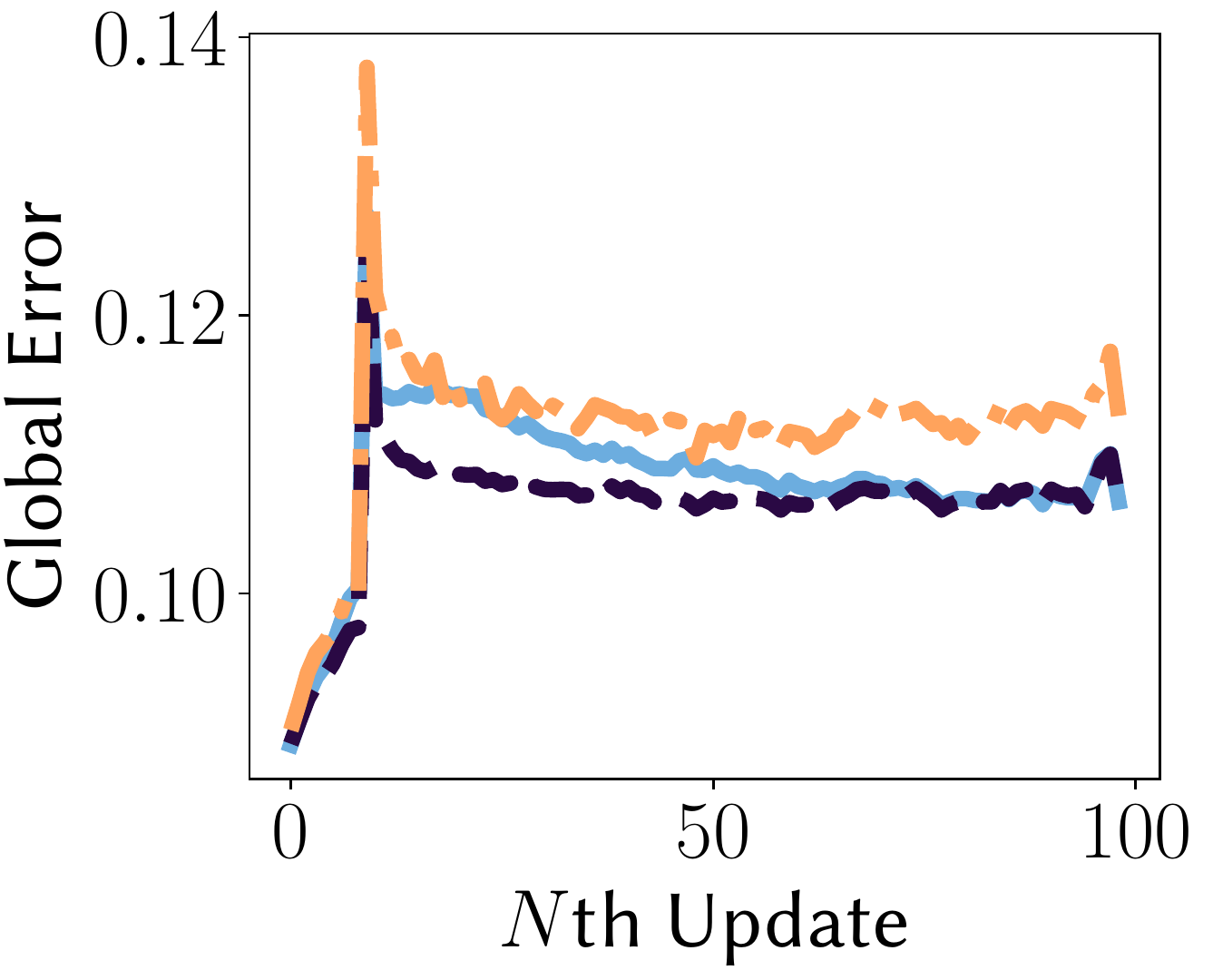}\label{fig:globalerror_us}}
	 \subfloat[Global Error on KR Stock]{\includegraphics[width=0.16\textwidth]{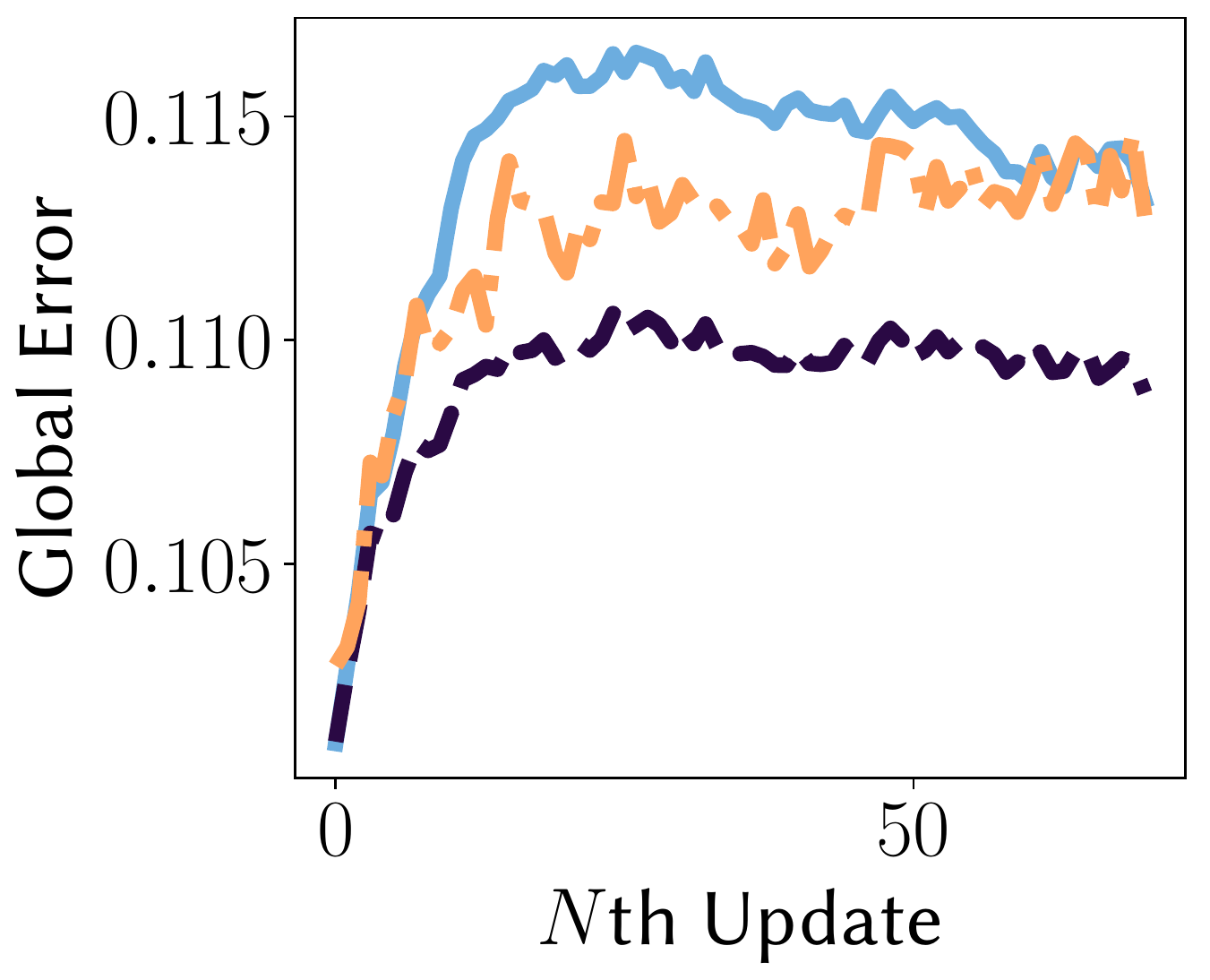}\label{fig:globalerror_kr}}
	 \subfloat[Global Error on JPN Stock]{\includegraphics[width=0.16\textwidth]{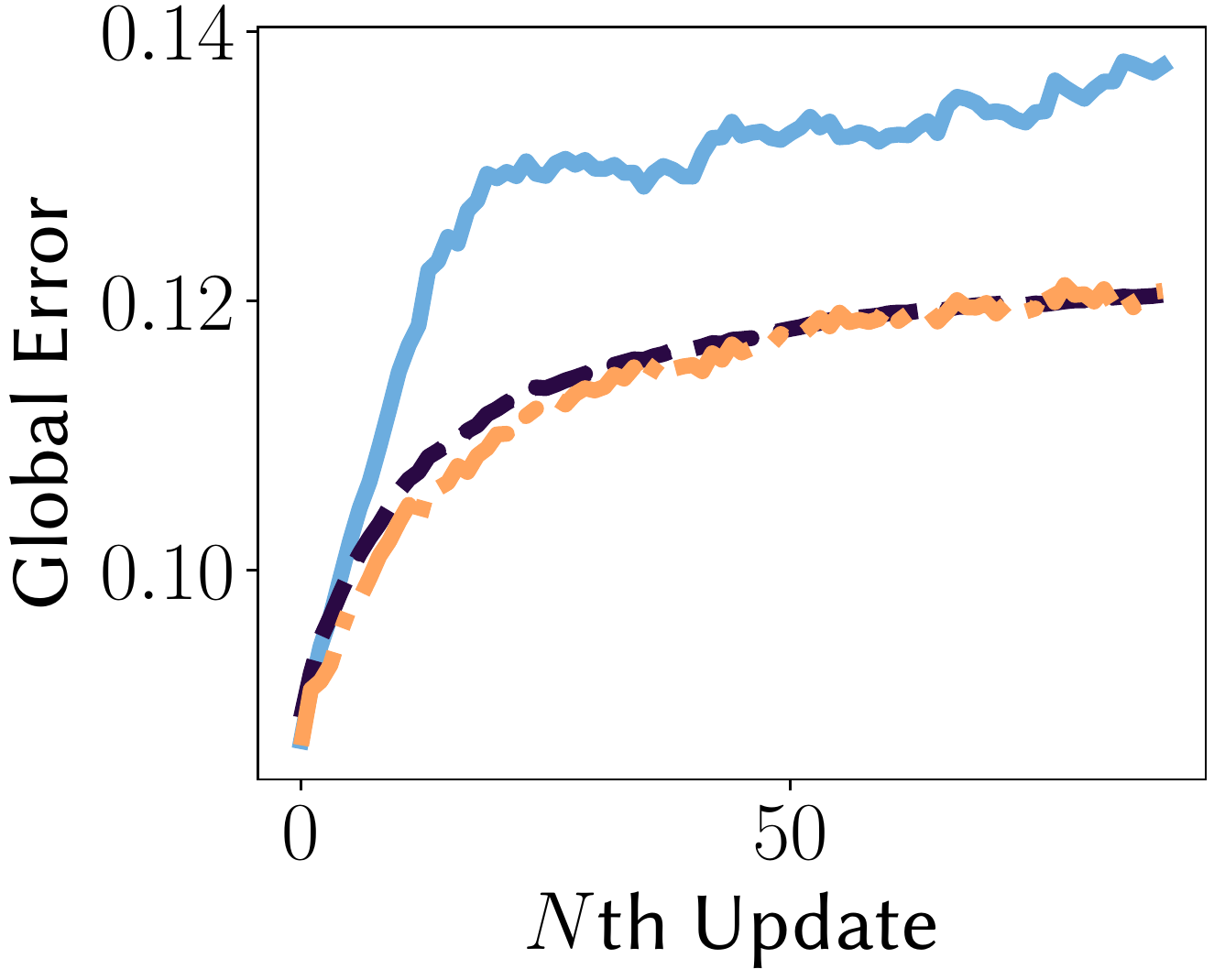}\label{fig:globalerror_jpn}}
	 \subfloat[Global Error on CHN Stock]{\includegraphics[width=0.16\textwidth]{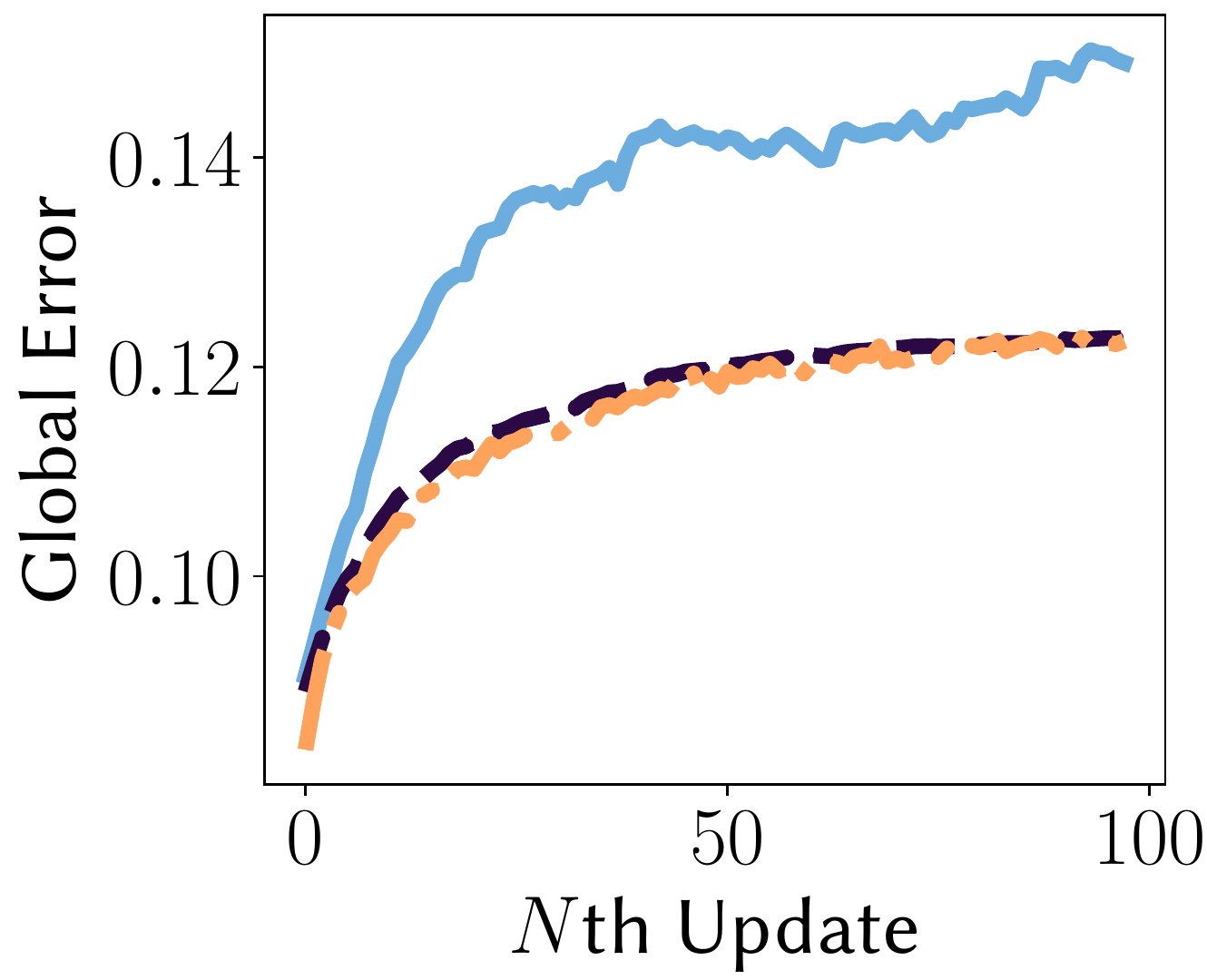}\label{fig:globalerror_chn}}	
	 \subfloat[Global Error on VicRoads]{\includegraphics[width=0.16\textwidth]{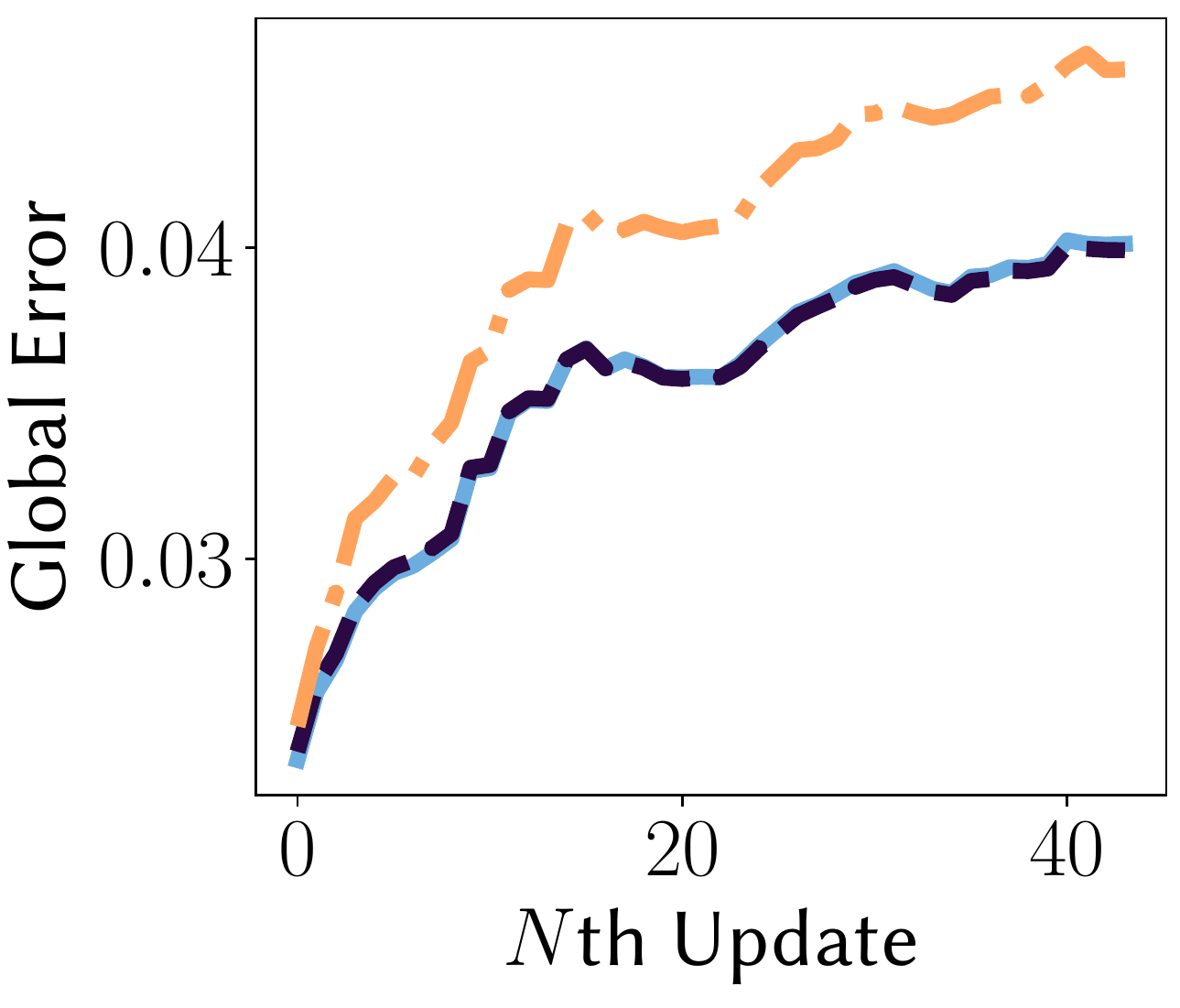}\label{fig:globalerror_traffic}}
	 \subfloat[Global Error on PEMS]{\includegraphics[width=0.16\textwidth]{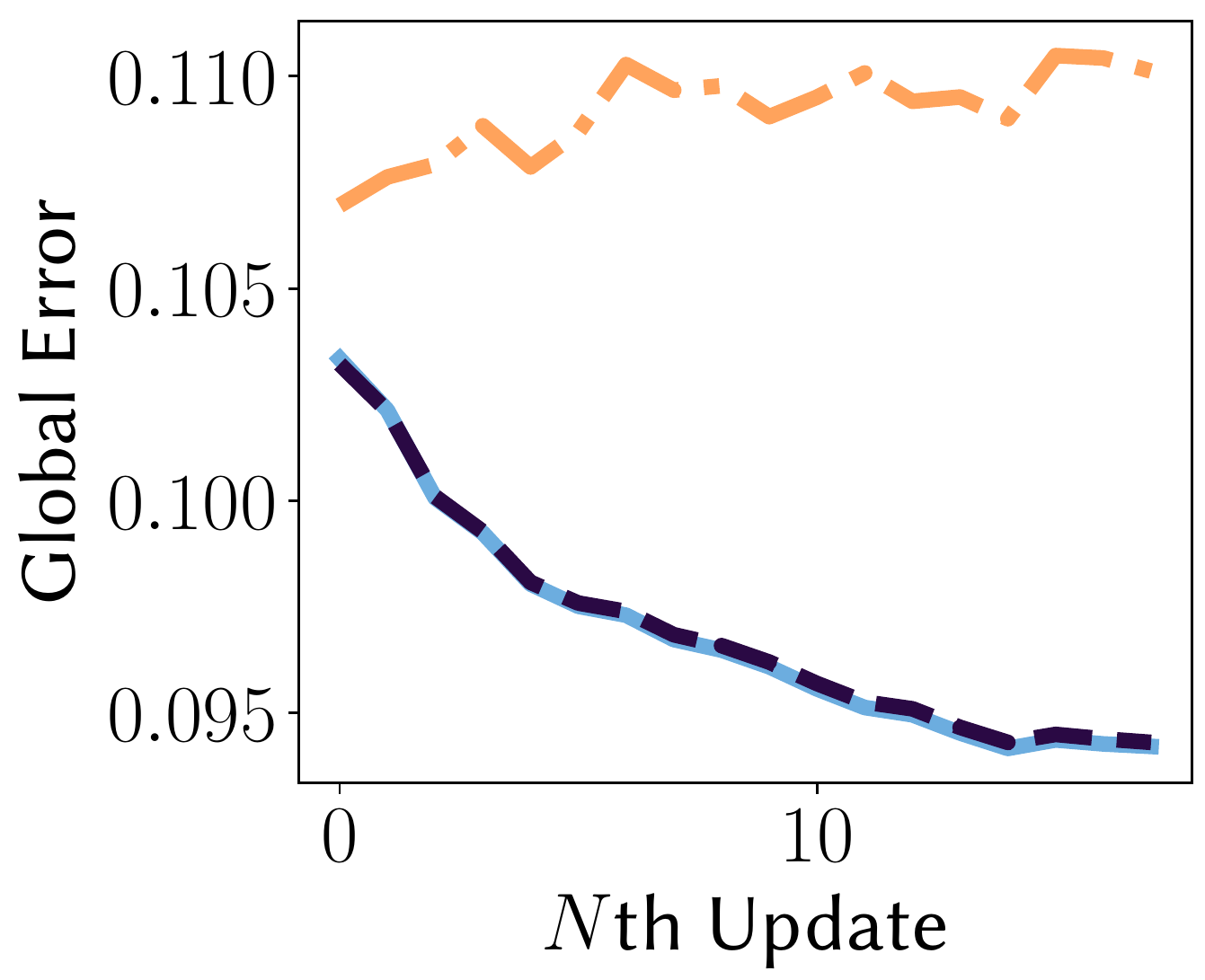}\label{fig:globalerror_pems}}	\\	 	
	\caption{Global error comparison between \method and \spartan.
	The other competitors have similar or high errors than that of \spartan.
	Although \method updates factor matrices using newly arrived data and pre-existing factor matrices, the global errors of \method with $\lambda = 1$ are lower than or equal to those of \spartan.
	}
	\label{fig:global_error}
\end{figure*}

\begin{figure}
	\centering	
	 \subfloat[Data Size in US Stock data]{\includegraphics[width=0.23\textwidth]{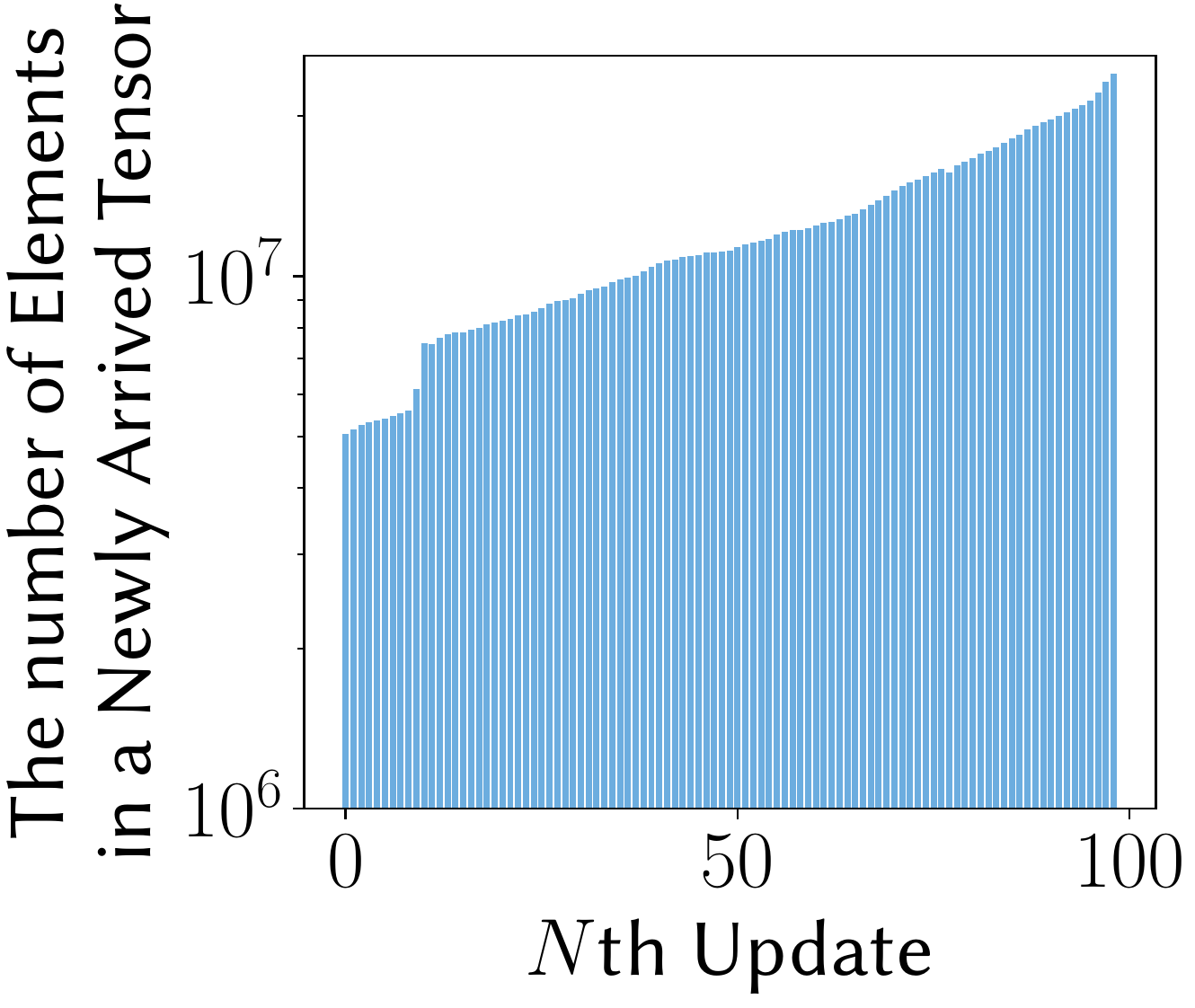}\label{fig:data_size_us}}	
	 \subfloat[Data Size in KR Stock data]{\includegraphics[width=0.23\textwidth]{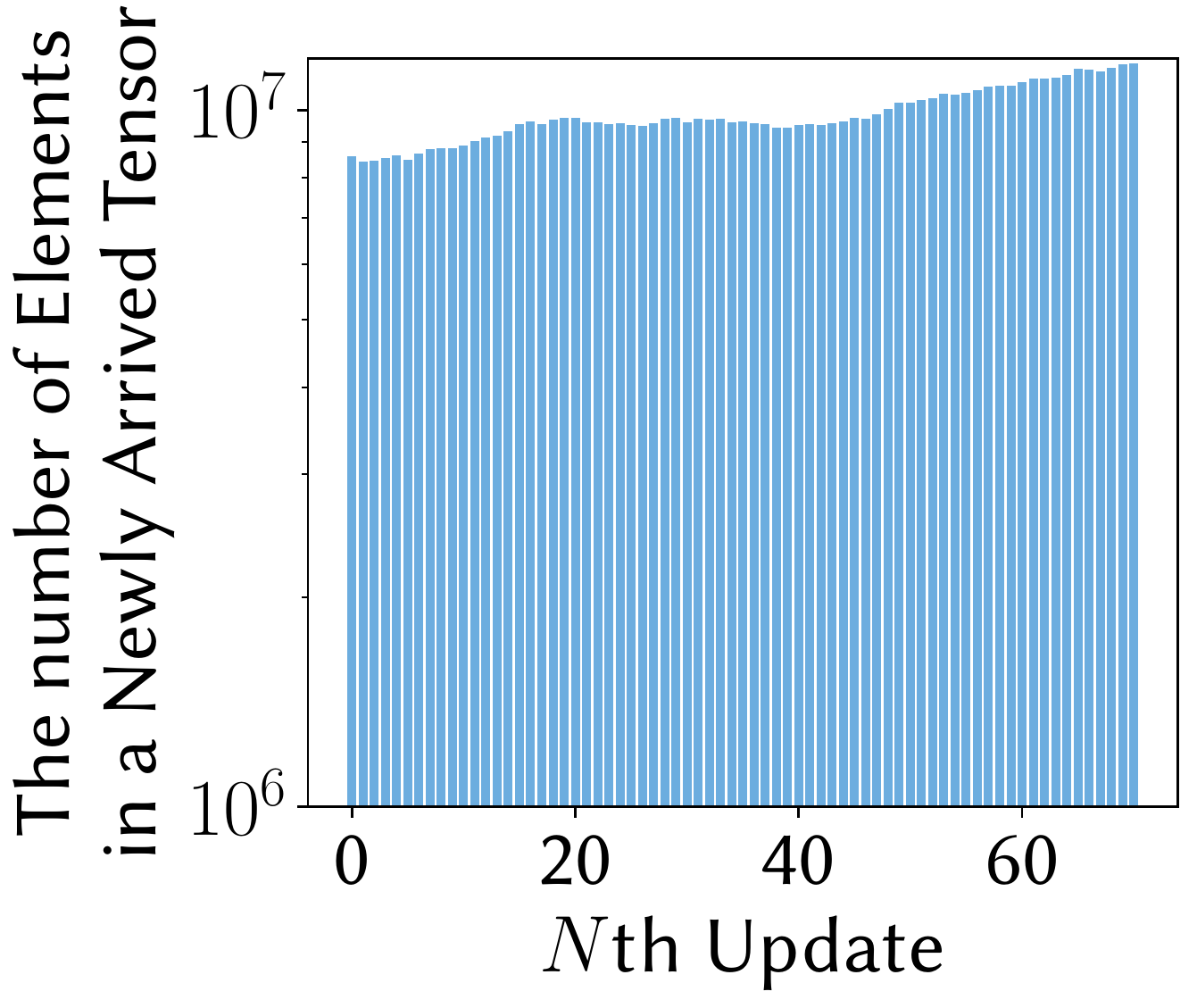}\label{fig:data_size_kr}}	
	 	 \\
	\caption{
	The number of elements in newly arrived data increases over updates on two datasets.
	}
	\label{fig:data_size}
\end{figure}

\section{Proofs}
We provide proofs of Lemmas and Theorem described in Section~\ref{sec:proposed}.

\subsection{Proof of Lemma~\ref{lemma:update_u}}
\label{subsec:proof_lemmaU}
\begin{proof}
	To obtain $\mat{U}_{k,new}$, we first derive $\frac{\partial \T{L}_{\mat{U}_{k,new}}}{\partial \mat{U}_{k,new}}$ as follows:
	\begin{align}
		\frac{\partial \T{L}_{\mat{U}_{k,new}}}{\partial \mat{U}_{k,new}} = -2\left( \mat{X}_{k,new}\mat{V}\mat{S}_k - \mat{U}_{k,new}\mat{S}_k\mat{V}^T\mat{V}\mat{S}_k \right)
	\end{align}
where $\T{L}_{\mat{U}_{k,new}}$ is described in Eq.~\eqref{eq:new_lossU}.
Then, we set $\frac{\partial \T{L}_{\mat{U}_{k,new}}}{\partial \mat{U}_{k,new}}$ to zero, and then arrange the terms by considering whether the terms includes $ \mat{U}_{k,new}$.
\begin{align}
	\mat{U}_{k,new}\mat{S}_k\mat{V}^T\mat{V}\mat{S}_k = \mat{X}_{k,new}\mat{V}\mat{S}_k
\end{align}

We finally obtain Eq.~\eqref{eq:updateU_lemma} by multiplying the inverse term $(\mat{S}_k\mat{V}^T\mat{V}\mat{S}_k)^{-1}$ to both sides.
\end{proof}

\subsection{Proof of Lemma~\ref{lemma:update_s}}
\label{subsec:proof_lemmaS}
\begin{proof}
	To obtain $\mat{W}(k,:)$, we first re-express the loss function:
	\begin{align}
		&\T{L} = \lambda \sum_{k=1}^{K} {\| vec(\mat{X}_{k,old})^T - \mat{W}(k,:)(\mat{V} \odot \mat{U}_{k,old})^T \|_F^2} \\
		& + \sum_{k=1}^{K+L} {\| vec(\mat{X}_{k,new})^T - \mat{W}(k,:)(\mat{V} \odot \mat{U}_{k,new})^T \|_F^2}
	\end{align}
	Then, we derive $\frac{\partial \T{L}_{\mat{W}(k,:)}}{\partial \mat{W}(k,:)}$ from the above function as follows:
	\begin{align}
	\begin{split}
		&\frac{\partial \T{L}_{\mat{W}(k,:)}}{\partial \mat{W}(k,:)} = -2\lambda(\mat{V} \odot \mat{U}_{k,old})^T (vec(\mat{X}_{k,old}) - (\mat{V} \odot \mat{U}_{k,old})\mat{W}(k,:)^T) \\
		& -2(\mat{V} \odot \mat{U}_{k,new})^T (vec(\mat{X}_{k,new}) - (\mat{V} \odot \mat{U}_{k,new})\mat{W}(k,:)^T) \\
		& = -2 \left(\lambda(\mat{V} \odot \mat{U}_{k,old})^T vec(\mat{X}_{k,old}) + (\mat{V} \odot \mat{U}_{k,new})^T vec(\mat{X}_{k,new})\right) \\
		& + 2 \bigg( \lambda(\mat{V} \odot \mat{U}_{k,old})^T(\mat{V} \odot \mat{U}_{k,old})\mat{W}(k,:)^T \\
		& + (\mat{V} \odot \mat{U}_{k,new})^T(\mat{V} \odot \mat{U}_{k,new})\mat{W}(k,:)^T \bigg)
		\end{split}
	\end{align}
	Then, we arrange the terms after setting $\frac{\partial \T{L}_{\mat{W}(k,:)}}{\partial \mat{W}(k,:)}=0$:
	\begin{align}
	\begin{split}
		& 2 \left( \lambda(\mat{V}^T\mat{V} * \mat{U}_{k,old}^T\mat{U}_{k,old}) + (\mat{V}^T\mat{V} * \mat{U}_{k,new}^T\mat{U}_{k,new})\right)\mat{W}(k,:)^T \\
		& = 2 \left(\lambda(\mat{V} \odot \mat{U}_{k,old})^T vec(\mat{X}_{k,old}) + (\mat{V} \odot \mat{U}_{k,new})^T vec(\mat{X}_{k,new})\right) \\
		\end{split}
	\end{align}	
	where $(\mat{A} \odot \mat{B})^T(\mat{A} \odot \mat{B})$ is equal to $(\mat{A}^T\mat{A} * \mat{B}^T\mat{B})$.
	We obtain Eq.~\eqref{eq:updateS_lemma} by multiplying the inverse term of $\lambda(\mat{V}^T\mat{V} * \mat{U}_{k,old}^T\mat{U}_{k,old}) + (\mat{V}^T\mat{V} * \mat{U}_{k,new}^T\mat{U}_{k,new})$ in both sides.
\end{proof}

\subsection{Proof of Lemma~\ref{lemma:update_v}}
\label{subsec:proof_lemmaV}
\begin{proof}
	To obtain $\mat{V}$, we derive $\frac{\partial \T{L}_{\mat{V}}}{\partial \mat{V}}$ as follows:
	\begin{align}
	\begin{split}
		\frac{\partial \T{L}_{\mat{V}}}{\partial \mat{V}}
		& = -2\lambda\sum_{k=1}^{K}{\left(\mat{X}_{k,old}^T\mat{U}_{k,old}\mat{S}_{k} - \mat{V}\mat{S}_{k}\mat{U}_{k,old}^T\mat{U}_{k,old}\mat{S}_{k}\right)} \\
		& -2\sum_{k=1}^{K+L}{\left(\mat{X}_{k,new}^T\mat{U}_{k,new}\mat{S}_{k} - \mat{V}\mat{S}_{k}\mat{U}_{k,new}^T\mat{U}_{k,new}\mat{S}_{k}\right)} \\
		& = -2 \left(\lambda\sum_{k=1}^{K}{\mat{X}_{k,old}^T\mat{U}_{k,old}\mat{S}_{k}} + \sum_{k=1}^{K+L}{\mat{X}_{k,new}^T\mat{U}_{k,new}\mat{S}_{k}} \right) \\
		& + 2\mat{V}\left(\lambda\sum_{k=1}^{K}{\mat{S}_{k}\mat{U}_{k,old}^T\mat{U}_{k,old}\mat{S}_{k}} + \sum_{k=1}^{K+L}{\mat{S}_{k}\mat{U}_{k,new}^T\mat{U}_{k,new}\mat{S}_{k}} \right)
	\end{split}		
	\end{align}
	Then, we set $\frac{\partial \T{L}_{\mat{V}}}{\partial \mat{V}} = 0$, and then arrange the terms as follows:
	\begin{align}
		\begin{split}
			& \mat{V}\left(\lambda\sum_{k=1}^{K}{\mat{S}_{k}\mat{U}_{k,old}^T\mat{U}_{k,old}\mat{S}_{k}} +\sum_{k=1}^{K+L}{\mat{S}_{k}\mat{U}_{k,new}^T\mat{U}_{k,new}\mat{S}_{k}} \right) \\
			& = \left(\lambda\sum_{k=1}^{K}{\mat{X}_{k,old}^T\mat{U}_{k,old}\mat{S}_{k}} + \sum_{k=1}^{K+L}{\mat{X}_{k,new}^T\mat{U}_{k,new}\mat{S}_{k}} \right)
		\end{split}
	\end{align}
	We obtain Eq.~\eqref{eq:updateV_lemma} by multiplying the inverse term of $\lambda$ $\sum_{k=1}^{K}\mat{S}_{k}\mat{U}_{k,old}^T$ $\mat{U}_{k,old}\mat{S}_{k} +\sum_{k=1}^{K+L}{\mat{S}_{k}\mat{U}_{k,new}^T\mat{U}_{k,new}\mat{S}_{k}}$ in both sides.
\end{proof}

\subsection{Proof of Theorem~\ref{theorem:time_complexity}}
\label{subsec:proof_time_complexity}
\begin{proof}
The overall time complexity of \method is proportional to the summation of the computational costs related to Eq.~\eqref{eq:updateU_lemma},~\eqref{eq:updateS_lemma}, and~\eqref{eq:updateV_lemma}.
Updating $\mat{U}_{k,new}$ for all $k$ takes $\T{O}(JR\sum_{k=1}^{K+L}{I_{k,new}})$ time since the dominant term is $\mat{X}_{k,new}\mat{V}$ in Eq.~\eqref{eq:updateU_lemma} where the sizes of $\mat{X}_{k,new}$ and $\mat{V}$ are $I_{k,new}\times J$ and $J \times R$, respectively.
In Eq.~\eqref{eq:updateS_lemma}, the dominant term $vec(\mat{X}_{k,new})^T(\mat{V} \odot \mat{U}_{k,new})$ requires $\T{O}(I_{k,new}JR)$, and thus updating $\mat{S}_{k}$ takes $\T{O}(JR\sum_{k=1}^{K+L}{I_{k,new}})$.
Updating $\mat{V}$ requires $\T{O}(JR\sum_{k=1}^{K+L}{I_{k,new}})$ since the dominant term in Eq.~\eqref{eq:updateV_lemma} is $\sum_{k=1}^{K+L}$ ${\mat{X}_{k,new}^T\mat{U}_{k,new}}$ where the sizes of $\mat{X}_{k,new}$ and $\mat{U}_{k,new}$ are $I_{k,new}\times J$ and $I_{k,new} \times R$, respectively.
Hence, the overall time complexity is $\T{O}(JR\sum_{k=1}^{K+L}{I_{k,new}})$.
\end{proof}

\section{Size of Newly Arrived Data}
\label{appendix:tensor_size}

Figure~\ref{fig:data_size} shows the size of newly arrived data at each update for US and KR Stock datasets.
In the two datasets, the size of newly arrived data increases over updates.
This is because new slice matrices come in addition to new rows of existing matrices.
In the other datasets, the size of newly arrived data is almost the same since only new rows of existing slice matrices arrive.

\section{Implementation of \spade}
\label{appendix:implement_spade}
Since \spade is designed only for handling new slice matrices, it is limited to directly use the implementation code of \spade in a dual-way setting.
Therefore, we modify the code of \spade to work in the setting.
\spade updates factor matrices with two steps:
1) updates factor matrices for new rows of existing slice matrices, and
2) updates factor matrices for new slice matrices.
We use the initialization code of \spade and the update code of \spade for the first and the second steps, respectively.
Note that the initialization code performs PARAFAC2 decomposition for a tensor consisting of accumulated existing slice matrices.

\section{Global Error Comparison}
\label{appendix:errors}
We compare global errors of \method with those of \spartan.
The other competitors have similar or high errors than that of \spartan.
Figure~\ref{fig:global_error} shows that \method has competitive global errors to \spartan.
When $\lambda$ is equal to $0.7$, \method has lower global errors than \spartan on US Stock, VicRoads, and PEMS datasets while \method has higher global errors than \spartan on KR Stock, JPN Stock, and CHN Stock datasets.
When $\lambda$ is equal to $1.0$, the global errors of \method are lower than or equal to those of \spartan.
If we consider fitting factor matrices more to an accumulated tensor, we use a large forgetting factor approximately equal to $1$.